%% file: main_arxiv.tex
\documentclass[11pt]{article}

\usepackage[utf8]{inputenc} 
\usepackage[T1]{fontenc}    
\usepackage{hyperref}       
\usepackage{url}            
\usepackage{booktabs}       
\usepackage{amsfonts}       
\usepackage{nicefrac}       
\usepackage{microtype}      
\usepackage{xcolor}
\usepackage{mathtools}
\usepackage{algorithm}
\usepackage{algorithmic}
\usepackage{amssymb}
\usepackage{enumitem}
\usepackage{calc}
\usepackage{amsthm}
\usepackage{subcaption}
\usepackage{graphicx}
\usepackage{authblk}

\usepackage{float}

\usepackage{wrapfig}
\definecolor{dbblue}{RGB}{10,65,155}

\hypersetup{colorlinks,urlcolor=dbblue,citebordercolor=dbblue,linkbordercolor=dbblue,filecolor=dbblue,filebordercolor=dbblue,citecolor=dbblue,linkcolor=dbblue,colorlinks=true}

\usepackage[verbose=true,letterpaper]{geometry}

\renewcommand{\baselinestretch}{1.25}

\usepackage[numbers]{natbib}

\newtheorem{thm}{Theorem}
\newtheorem{lem}{Lemma}

\newtheorem{ass}{Assumption}
\newtheorem{prop}{Proposition}
\newtheorem{rem}{Remark}

\def\mnn{m^{\text{NN}}}
\def\msub{m^{\text{KNNR}}}

\title{$K$-Nearest-Neighbor Resampling for Off-Policy Evaluation in Stochastic Control\footnotetext{\scriptsize{\hskip-0.55cm Michael Giegrich (\href{mailto:giegrich@maths.ox.ac.uk}{giegrich@maths.ox.ac.uk}) is supported by the EPSRC Centre for Doctoral Training in Mathematics of Random Systems: Analysis, Modelling and Simulation (EP/S023925/1). Roel Oomen is employed as a Managing Director of FIC quantitative trading at Deutsche Bank A.G. This paper was prepared within the Sales and Trading function of DB, and was not produced, reviewed or edited by the DB Research Department. The views and opinions rendered in this paper reflect the author's personal views about the subject. No part of the author's compensation was, is, or will be directly related to the views expressed in this paper.}}}

\author[a]{Michael Giegrich}
\affil[a]{University of Oxford}
\affil[b]{Deutsche Bank, London}
\affil[c]{London School of Economics}
\affil[d]{Oxford Man Institute of Quantitative Finance}
\author[b,c]{Roel Oomen}
\author[a,d]{Christoph Reisinger}

\date{June 2023}

\begin{document}

\maketitle

\begin{abstract}
In this paper, we propose a novel $K$-nearest neighbor resampling procedure for estimating the performance of a policy from historical data containing realized episodes of a decision process generated under a different policy. We provide statistical consistency results under weak conditions. In particular, we avoid the common assumption of identically and independently distributed transitions and rewards. Instead, our analysis allows for the sampling of entire episodes, as is common practice in most applications. To establish the consistency in this setting, we generalize Stone's Theorem, a well-known result in nonparametric statistics on local averaging, to include episodic data and the counterfactual estimation underlying off-policy evaluation (OPE). By focusing on feedback policies that depend deterministically on the current state in environments with continuous state-action spaces and system-inherent stochasticity effected by chosen actions, and relying on trajectory simulation similar to Monte Carlo methods, the proposed method is particularly well suited for stochastic control environments. Compared to other OPE methods, our algorithm does not require optimization, can be efficiently implemented via tree-based nearest neighbor search and parallelization, and does not explicitly assume a parametric model for the environment's dynamics. Numerical experiments demonstrate the effectiveness of the algorithm compared to existing baselines in a variety of stochastic control settings, including a linear quadratic regulator, trade execution in limit order books, and online stochastic bin packing.
\end{abstract}
\vfill
\thispagestyle{empty}
\pagebreak\setcounter{page}{1}

\input{20230503intro}

\input{20230504_algo}

\input{20231121theory_main}

\input{20230506experiments_main}

\input{20231124proofs}

\input{20230507conclusion}

\begin{small}
\renewcommand{\baselinestretch}{1}
\bibliography{refs}
\appendix
\input{20231121theory}
\input{20231126_implementation}

\input{20230203experiments}

\end{small}

\end{document}

%% file: 20230503intro.tex
\section{Introduction}
In reinforcement learning (RL), off-policy evaluation (OPE) deals with the problem of estimating the value of a target policy with observations generated from a different behavior policy. OPE methods are typically applied to sequential decision making problems where observational data is available but experimentation with the environment directly is not possible or costly. More broadly, OPE methods are a widely researched subject in RL (see, e.g., \cite{voloshin2021empirical, fu2021benchmarks,uehara2022review} for recent overviews), however, relatively little attention has been paid to stochastic environments where the stochasticity depends on the chosen actions
and state and action spaces are continuous. For example, common benchmark problems are either deterministic or have finite state and/or action spaces (see, e.g., \cite{voloshin2021empirical, fu2021benchmarks}).

Notwithstanding this, stochastic control problems are precisely concerned with the setting where a decision process affects random transitions. Stochastic control is a field closely related to reinforcement learning and its methods have been applied to a wide range of high-stakes decision-making problems in diverse fields such as operations research \cite{gallego1994optimal,martinez2021optimizing}, economics \cite{kendrick2005stochastic,hudgins2019stress}, electrical engineering \cite{mu2019energy,chen2016state}, autonomous driving \cite{williams2018information} and finance \cite{cartea2015optimal,schmidli2007stochastic}. In the stochastic control literature, optimal policies are often represented as deterministic feedback policies (i.e., as deterministic functions of the current state) and, in the episodic case, are non-stationary due to the impact of a finite time-horizon. 

Stochastic control environments pose a challenging setting for OPE methods. For example, classical methods like importance sampling (IS) \cite{precup2000eligibility} struggle with deterministic target policies in continuous action spaces due to the severe policy mismatch between the target and the behavior policy (see, e.g. \cite{mou2023kernel}, for more details on the difficulty of estimation in such a regime in general). This policy mismatch can lead to a high variance in the estimate (see, e.g., \cite{andradottir1995choice, precup2001off,dann2014policy,mahmood2014weighted, mahmood2017multi, schlegel2019importance}), which is exacerbated by the stochastic control environments' inherent stochasticity. Furthermore, IS requires knowledge of the distribution of the behavior policy. Learning this distribution,  known as imitation learning, in a continuous action space of a policy functionally depending on a continuous and potentially high-dimensional state adds a further complication to this class of methods. Lastly, OPE methods not based on importance sampling often assume for their theoretical guarantees that the data sets used for value estimation contain only identically distributed and independent (iid) transitions (see, e.g., Fitted Q Evaluation \cite{le2019batch},  minimax V learning \cite{feng2019kernel}, model-free Monte Carlo \cite{fonteneau2010model,fonteneau2013batch}, DualDice \cite{nachum2019dualdice}). However, for episodic control problems this is an unrealistic assumption, as previous interaction with the environment should also be episodic and not rely on a potentially non-existent oracle that generates iid transitions. 

In this paper, we propose a novel procedure for off-policy evaluation based on a nearest-neighbor resampling algorithm. Assuming continuity in the rewards and the state transitions, our procedure exploits that similar state/action pairs (in a metric sense) are associated with similar rewards and state transitions. This addresses the counterfactual estimation problem of evaluating the performance of a target policy using a data set sampled under a different behavior policy without knowing or estimating the distribution of this policy. Similar to Monte Carlo methods, the resampling algorithm simulates trajectories and has a provably vanishing mean squared error (MSE) under episodic sampling. Compared to other OPE methods, our algorithm does not require optimization, can be efficiently implemented via tree-based nearest neighbor search and parallelization and does not explicitly assume a parametric model for the environment's dynamics. These properties make the proposed resampling algorithm particularly useful and directly applicable for stochastic control environments.
\setlength{\tabcolsep}{5pt}
\begin{table}[tbhp]
	{\footnotesize
		\caption{In this table, we summarize statistical results on commonly used model-based OPE methods. We report whether the results take into account episodic sampling, whether they apply to continuous state/ action spaces, whether model-inherent stochasticity was considered and whether statistical consistency was proved.}\label{tab:consistency}
		\begin{center}
			\begin{tabular}{|c||c|c|c|c|c|} \hline
				Method & Episodic & Cont. State & Cont. Action&  Stochastic & Consistency\\ \hline
				
				Mannor et al. 2007 \cite{mannor2007bias} & No & No & No& Yes &No \\
				Fonteneau et al. 2010 \cite{fonteneau2010model}  & No & Yes & Yes& Yes &No\\
				Liu et al. 2018 \cite{liu2018representation} &Yes &Yes&No&Yes &No\\
				Gottesman et al. 2019 \cite{gottesman2019combining} & Yes & Yes & No\footnote{See Section "Related Work".} & No & Partially\\
				Nachum et al. 2019 \cite{nachum2019dualdice} & No &Yes &Yes &Yes & Partially\\
				Yin \& Wang 2020 \cite{yin2020asymptotically} & Yes &No &No  &Yes & Yes\\
				Uehara \& Sun 2021 \cite{uehara2021pessimistic} & No & Yes &Yes &Yes & Yes\\
				Wang et al. 2024 \cite{wang2024off} & Partially & No & No& Yes &Yes \\ \hline
			\end{tabular}
		\end{center}
	}
\end{table}
\paragraph{Contributions} 
\begin{enumerate}
	\item We introduce a novel nearest-neighbor based resampling algorithm for off-policy evaluation in stochastic control environments that can generate trajectories without exact knowledge of the behavior policy and does not rely on estimating a parametric model for the dynamics (see Section \ref{sec:algo}). It is faster than comparable existing algorithms and parallelizable. 
	\item We prove that the MSE of this method vanishes for episodic sampling, under only weak assumptions
	(Section \ref{sec:theory_main}). 
	For model-based OPE methods, this is to our knowledge the first consistency result under episodic sampling for non-deterministic environments with continuous state-action spaces. An overview of statistical results for model-based OPE methods can be found in Table \ref{tab:consistency}.
	\item To prove consistency of the resampling algorithm, we extend a well known statistical result known as Stone's Theorem (see \cite{gyorfi2002distribution}, Theorem 4.1), which gives conditions for the consistency of locally weighted regressions. Our generalization consists of including the episodic setting and the counterfactual element of OPE. This result can be found in Theorem \ref{thm:stone_mult} in Section \ref{sec:theory_main}. 
	\item We demonstrate the method's effectiveness on various stochastic control environments including a linear quadratic regulator, trade execution in limit order books and online stochastic bin packing. These environments vary in their control mechanisms, sources of uncertainty, areas of application and dimensionality. We show competitive performance of the resampling algorithm compared to baselines (see Section \ref{sec:experiments}). 
\end{enumerate}

\paragraph{Related Work} Our method is most closely related to the model-free Monte Carlo (MFMC) algorithm proposed in \cite{fonteneau2010model,fonteneau2013batch}, where a nearest neighbor method is used for trajectory generation. Under fairly generic assumptions, this method leads to a computational complexity of $\mathcal{O}(n^2)$ in the number of observed sampling episodes $n$. In contrast, our method can utilize efficient tree-based nearest neighbor search algorithms, which leads to a more favourable computational complexity of $\mathcal{O}(K_n n\log(n))$ where the number of nearest neighbours under consideration is of the order $o(n)$ \footnote{In our experiments, we choose $K_n$ in the order of $\mathcal{O}(n^{\frac{1}{4}})$ based on empirical observations. A heuristic choice for the nearest neighbor growth rate is common (i.e., \cite{lall1996nearest}) and regression theory suggests a rate smaller than $1/2$ (i.e., Theorem 6.2 in \cite{gyorfi2002distribution} suggests the rate $\mathcal{O}(n^{\frac{2}{d+2}})$ where $d\geq3$ is the dimension of the feature space).}. This translates into improvements in computational costs of an order of  magnitude even for moderately-sized data sets (see Section \ref{sec:experiments}). Furthermore, our algorithm allows for almost complete parallelization in the generation of trajectories. On the theory side, \cite{fonteneau2010model,fonteneau2013batch} provide point-wise error bounds on the bias and variance depending on the distances of the nearest neighbors, while not studying convergence properties. In contrast, we establish our algorithm's consistency, which is global in the state domain. Crucially, their results only hold for data sets containing iid transitions, while our results allow for data sets sampled in episodes. Note that our analytic approach can be modified to establish consistency for the MFMC. A more detailed comparison between our resampling algorithm and MFMC can be found in the supplementary material.

In \cite{gottesman2019combining}, a version of the MFMC algorithm was combined with a parametric model approach to improve performance for smaller data sets. Consistency is derived for a deterministic environment by using the stability of the deterministic dynamical system, and under the critical assumption that for any fixed action the distance to the nearest neighbor state vanishes as more data becomes available\footnote{In statistics, this is commonly considered a statement that should be proven (e.g., \cite{gyorfi2002distribution}, Lemma 6.1).}. Crucially, they only consider nearest neighbor matching in the state space and require exact matches in the action space (Assumption 1 in \cite{gottesman2019combining}). A simple probabilistic argument shows, that this assumption excludes data generated by not fully atomic behavior policies on continuous action spaces and practically limits the results to discrete action spaces.

Finally, there is a broader literature on nearest neighbor search and other local weighting approaches in the context of learning for dynamical systems. In \cite{lall1996nearest}, the authors introduced a nearest neighbor resampling technique similar to ours for simulating hydrological time series. Their method has attracted considerable interest in hydrology, where several other papers apply it in this context (e.g.,\cite{rajagopalan1999k,towler2009simulating,raseman2020nearest}). This literature, however, does not consider interventions and counterfactual effects, and does not provide any theoretical analysis of the algorithm. Other examples of nearest neighbor methods in reinforcement learning include imitation learning \cite{pari2022surprising}, memory-based learning \cite{moore1990efficient,baird1993reinforcement,atkeson1997locally,atkeson1997locally_control,mansimov2018simple,humphreys2022large}, and data augmentation \cite{li2022associative}.

%% file: 20230504_algo.tex
\section{Description of Algorithm}\label{sec:algo}
\paragraph{Notation} For a metric space $(S,d)$, $C(S,S)$ will denote the space of continuous functions with respect to $d$ from $S$ onto $S$. We will use the following multi-index notation: $i^{k}=(i_0,\ldots,i_k)$,  $t^k = (t_0,\ldots,t_k)$, in particular $(i^{k},t^k)=((i_0,t_0),\ldots,(i_k,t_k))$. For any two multi-indexes $i^k$ and $j^l$ with $l\in\mathbb{N}$,  denote $[i^k,j^l]=(i_0,\cdots,i_k,j_0\cdots,j_l)\in\mathbb{N}^{k+l+2}$ and, for $l<k$, $i^{l:k}=(i_{l},\cdots,i_k)\in\mathbb{N}^{k-l+1}$. We define the summation $\sum_{i^{k}, t^k}^{n, T} \coloneqq \sum_{i_1=1}^{n}\sum_{t_1=0}^{T-1}\cdots\sum_{i_k=1}^{n}\sum_{t_k=0}^{T-1}$, where $T,n\in\mathbb{N}$. The operator $|\cdot|$ applied to a set denotes its cardinality. If applied to a multindex, $|i^k|$ denotes the cardinality of the set $\{i_j\}_{j=0}^k$. $I_{\cdot}$ defines an indicator function. For any set $S$, $\mathcal{V}(S)$ denotes the uniform distribution on $S$.
 

\paragraph{Setting} We are given a data set $\hat{\mathcal{D}}\coloneqq\{(X_i,U_i,R_i)\}_{i=1,...,n}$, where the state process $X_i = (X_{i,t})_{t=0,...,T}$, the action process $U_i = (U_{i,t})_{t=0,...,T}$ and the reward process $R_i = (R_{i,t})_{t=0,...,T}$ are stochastic processes on $\mathcal{X}\subset \mathbb{R}^{d_1}$, $\mathcal{U}\subset\mathbb{R}^{d_2}$ and $\mathbb{R}$, respectively, with $d_1,d_2\in\mathbb{N}$. Here, $\mathcal{X}$ denotes the state space and $\mathcal{U}$ the action space. For each $i$, $((X_{i,t},U_{i,t},R_{i,t}))_{t=0,...,T}$ is called a sample episode and we assume that $\hat{\mathcal{D}}$ contains $n$ episodes, which are rolled out independently and with the same behavior policy. For all $t$, the  behavior policy is given by the distribution of $U_{i,t}$ conditioned on $X_{i,t}$. Furthermore, denote the data set without the last period as $\mathcal{D}\coloneqq\hat{\mathcal{D}}\setminus{\{(X_{i,T},U_{i,T},R_{i,T})\}_{i=1,...,n}}$.\footnote{While our algorithm and our theoretical results could be formulated slightly more data efficient if we use $\hat{\mathcal{D}}$ instead of $\mathcal{D}$, we omit this detail to simplify notation.}

Each episode is generated by a random dynamical system where we use a similar formulation as in \cite{bertsekas1996stochastic,bhattacharya2003random,bertsekas2019reinforcement}. Consider $S \coloneqq\mathcal{X}\times\mathcal{U}$ equipped with a metric $d$ and let $\mathcal{S}$ be the Borel $\sigma$-algebra of $S$. Let $\Gamma\subset C(S,S)$, a set of uniformly equicontinuous functions, be endowed with a $\sigma$-algebra $\Sigma$ such that $(\gamma, x) \rightarrow \gamma(x)$ is $(\Gamma \times S, \Sigma \otimes \mathcal{S})$-$(S, \mathcal{S})$ measurable and let $\mathbb{Q}$ be a probability measure on $(\Gamma, \Sigma)$. On a probability space $(\Omega, \mathcal{F}, \mathbb{P})$, let $\left(\alpha_t\right)_{t=0}^{T-1}$ be a sequence of iid random functions from $\Gamma$ distributed according to $\mathbb{Q}$. For an initial $S$-valued random variable $(X_0,U_0)$, independent of the sequence $\left(\alpha_t\right)_{t=0}^{T-1}$, a sample episode is generated according to $(X_{t+1},U_{t+1}) = \alpha_{t} (X_t,U_t)$ or equivalently by $(X_{t+1},U_{t+1}) = \alpha_{t} \alpha_{t-1} \ldots \alpha_0 (X_0,U_0)$. We assume that the distribution of $X_0$ is known and we denote it by $\nu_0$. 

For the process $R=(R_t)_{t=0,...,T}$, we assume that there exists a sequence of iid real random variables $(\epsilon_t)_{t=0}^{T}$, independent of the system noise, distributed on $[-M, M]$ for $M\in\mathbb{R}$ with $\mathbb{E}[\epsilon_t]=0$, and a continuous function $r:S\rightarrow [-M,M]$ such that $R_t = r(X_t,U_t)+\epsilon_t$. For sample $i$, denote with $\epsilon_{i,t}$ and $\alpha_{i,t}$ the corresponding random variable and random function of this sample. We call this dynamical system the sampling system. While we are not aware of any general result relating  this set-up to Markov decision processes (MDP), we provide some discussion in the supplementary materials and note that the setting we consider here is more general compared than the settings considered in \cite{fonteneau2010model,fonteneau2013batch, gottesman2019combining}.

Furthermore, we define a second dynamical system, on the same space as the first system, as follows. Given a sequence of iid random functions $\left(\alpha_t\right)_{t=0}^{T-1}$ as above, consider a family of feedback policies $u=(u_t)_{t=0}^T$ with $u_t:\mathcal{X}\rightarrow\mathcal{U}$.  Define the sequence $\left(\alpha_t|^u\right)_{t=0}^{T-1}$ where $\alpha_t|^u =(\Pi(\alpha_t),u_t(\Pi(\alpha_t))) $ and $\Pi:\mathcal{X}\times\mathcal{U} \rightarrow\mathcal{X}$ with $\Pi(x,u)=x$. For an initial $S$-valued random variable $(X_0,u_0(X_0)))$, where $X_0$ is the same as in the sampling system and is independent of the sequence $\left(\alpha_t|^u\right)_{t=0}^{T-1}$, an episode is generated by $(\widehat{X}_{t+1},u_{t+1}(\widehat{X}_{t+1})) = \alpha_{t}|_u (\widehat{X}_t,\widehat{U}_t)$ or equivalently by $(\widehat{X}_{t+1},u_{t+1}(\widehat{X}_{t+1}))  = \alpha_{t}|_u \alpha_{t-1}|_u \ldots \alpha_0|_u (X_0,u_0(X_0))$. Let $(\epsilon_t)_{t=0}^{T-1}$ be a sequence of iid real random variables  
as above, then define the stochastic process $R|_u=(R_t|_u)_{t=0,...,T}$ by
$R_t|_u = r(\widehat{X}_t,u_t(\widehat{X}_t))+\epsilon_t$.

We call this dynamical system the counterfactual system and for a given realisation of the noise $(X_0, \left(\alpha_t\right)_{t=0}^{T-1}, \left(\epsilon_t\right)_{t=0}^{T})$, we call the generated episode of the counterfactual system the counterfactual path. Note that for each episode generated by the sampling system a counterfactual path can be identified.
Our algorithm has the target to consistently estimate the cumulative reward under the counterfactual system with target policy $u$:
\begin{equation}
	m(u)\coloneqq\mathbb{E}\Big[\sum_{t=0}^{T}R_t|_u\Big].
\end{equation}

 \begin{rem}
A commonly made simplifying assumption in the OPE literature is that data consist of iid transitions $\{X_{i},U_i,R_i,X_i'\}_{i=1}^n$, where $X_i'$ denotes the subsequent state \cite{fonteneau2010model,fonteneau2013batch,le2019batch,hao2021bootstrapping}. Note that the samples in our setting are not assumed iid, but may exhibit an episodic dependence structure, which poses a considerable complication in the analysis of OPE methods not relying on importance weights. However, considering episodic data is often more realistic as in most practical situations data is gathered in episodes rather than from one-transition oracles.
\end{rem}

\begin{rem}\label{rem:generalization}
Our algorithm and the corresponding consistency results allow for two generalizations of targets: First, we can allow the reward process to depend not only on the current state and action, but also the subsequent state, i.e., we can define the reward process as $\widetilde{R}=(\widetilde{R}_t)_{t=0,...,T}$
by a sequence $(\epsilon_t)_{t=0}^{T}$ as above and a continuous function $\widetilde{r}:\mathcal{X}\times\mathcal{U}\times\mathcal{X}\rightarrow [-M,M]$, such that $\widetilde{R}_t = \widetilde{r}(X_t,U_t,X_{t+1})+\epsilon_t$. Second, we can consistently estimate the expectation of functions of counterfactual paths. More precisely, for a known continuous function $g:\mathcal{X}^T\rightarrow[-M,M]$, we can consistently estimate $\mathbb{E}[g|_u]$, where $(\widehat{X}_t)_{t=0}^T$ is a counterfactual path and $g|_u=g(\widehat{X}_0,\dots, \widehat{X}_{T})$ a random variable. This extension allows for example the consistent estimation of path-dependent risk measures using our algorithm. Since the two extensions can be treated analogously to the original rewards, both in the algorithm and the main theorem, we omit them for readability.
\end{rem}

\paragraph{$K$-Nearest Neighbors and $K$-Nearest Neighbor Paths} Given a tuple $(x_0, v_0)\in S$, the reordered version of a data set $\mathcal{D}$ is given by $\{(X_{(i,nT)},U_{(i,nT)},R_{(i,nT)})(x_0,v_0)\}_{i=1}^{nT}$, where the ordering is according to increasing values of $d\left( (x_0,v_0),(X_{i,t},U_{i,t})\right)$. For simplicity, we assume for this ranking and all the subsequent ones that no ties occur. For $K\in\mathbb{N}$ with $K<n$, $(X_{i,t},U_{i,t})$ is a $K$-nearest neighbor of $(x_0,v_0)$ in $\mathcal{D}$, if for the corresponding sample in the reordered data set $(X_{(j,nT)},U_{(j,nT)})(x_0,v_0)= (X_{i,t},U_{i,t})$ it holds that $j\leq K$. For both the reordered data sets and nearest neigbhors, we omit the dependence on $(x_0,v_0)$ for readability, if it is clear from the context.

For $\tau=0,\dots,T$, define  a $K$-nearest neighbor ($K$-NN) path for a given initial state $x_0\in\mathcal{X}$, a feedback policy $u$, and data  $\mathcal{D}$ 
by the following construction. First, reorder $\mathcal{D}$ according to increasing values in $d\left( (x_0, u_0(x_0)),(X_{i,t},U_{i,t})\right)$ and consider the reordered version  $\{(X_{(i,Tn)},U_{(i,Tn)},R_{(i,Tn)})\}_{i=1}^{Tn}$. For $k_0\leq K $, identify the $k_0$-nearest neighbor  $(X_{(k_0,Tn)},U_{(k_0,Tn)})$ with the original sample from the data set and denote it by $(X^{(k_0)}_{i_0,t_0},U^{(k_0)}_{i_0,t_0})$.  Next, consider $(X^{(k_0)}_{i_0,t_0+1},u_1(X^{(k_0)}_{i_0,t_0+1}))$ and reorder the data set $\mathcal{D}$ according to increasing values in $d\big( (X^{(k_0)}_{i_0,t_0+1},u_1(X^{(k_0)}_{i_0,t_0+1})),(X_{i,t},U_{i,t})\big)$ and denote it by $\{(X^{(k_0)}_{(i,Tn)},U^{(k_0)}_{(i,Tn)},r^{(k_0)}_{(i,Tn)})\}_{i=1}^{Tn}$. For $ k_1\leq K $, again identify $(X^{(k_0)}_{(k_1,Tn)},U^{(k_0)}_{(k_1,Tn)})$ and denote it by $(X^{(k_0,k_1)}_{i_1,t_1},U^{(k_0,k_1)}_{i_1,t_1})$. We repeat this procedure $\tau$ times where for any $q=1,\dots,\tau$,{} $k^{q-1}$ and $k_q$ we search for the $k_q$-nearest neighbor of $(X^{k^{q-1}}_{i_{q-1},t_{q-1}+1},u_{q}(X^{k^{q-1}}_{i_{q-1},t_{q-1}+1}))$ in the set $\mathcal{D}$ and then identify the $k_q$-nearest neighbor with $(X^{k^{q}}_{i_{q},t_{q}},U^{k^{q}}_{i_{q},t_{q}})$. For an initial state $x_0\in \mathcal{X}$, $((X_{j_p,s_p}, U_{j_p,s_p}))_{p=0}^{\tau}$ is a $K$-NN path of length $\tau$ if there exists a sequence $k^{\tau}=(k_0,k_1,\dots,k_{\tau})$ with $k_q\leq K$ for all $q=0,\dots,\tau$ such that  $((X_{j_p,s_p}, U_{j_p,s_p}))_{p=0}^{\tau} =((X^{k^{q}}_{i_{q},t_{q}},U^{k^{q}}_{i_{q},t_{q}})((x_0, u_0(x_0)))_{q=0}^{\tau}$ . Equivalently, we say $(j^{\tau},s^{\tau})$ is a $K$-NN path. Since we assume that each nearest neighbor attribution is unique, we can also uniquely identify each $K$-NN path by the sequence $k^{\tau}$ for any initial $x_0$. Thus, for any initial $x_0$ each $K$-NN path can be uniquely written as $(X_{t}^{k^{\tau}}, U_{t}^{k^{\tau}}, R_{t}^{k^{\tau}})_{t=0,\dots,{\tau}}$. We provide a visualization of a K-NN path in Figure \ref{fig:knnPath}.

\begin{wrapfigure}{L}{0.65\textwidth}\label{fig:knnPath}
	\begin{center}
		\includegraphics[width=0.6\textwidth]{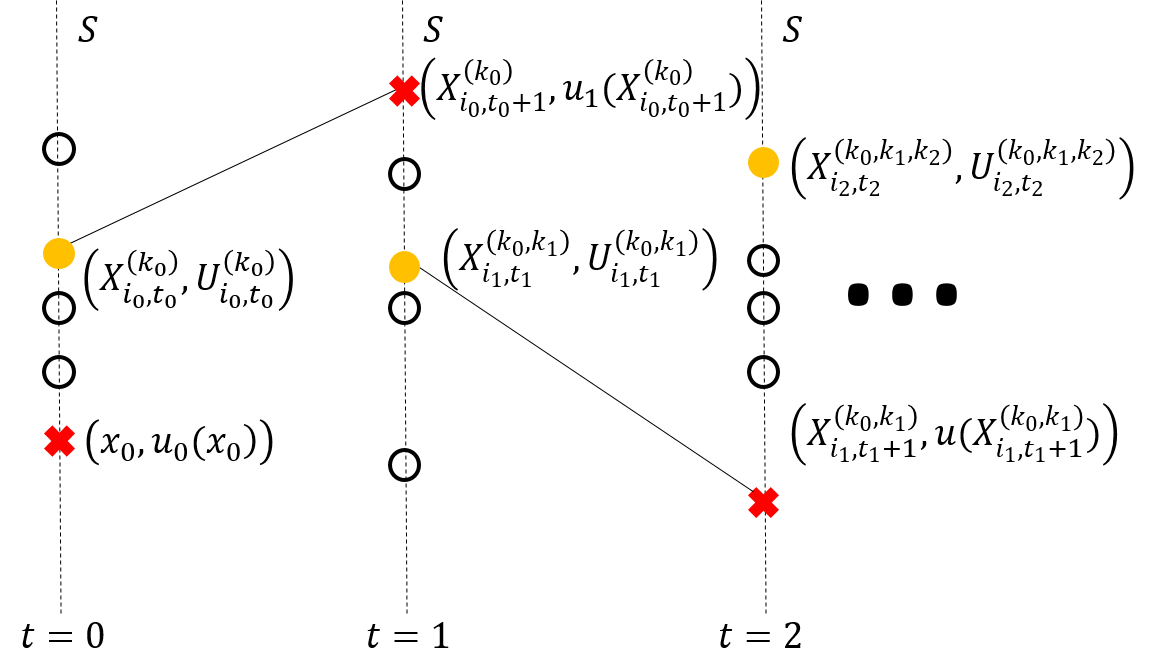}
	\end{center}
	\caption{Visualization of the $K$-NN path with $k^3=(3,2,4)$ where $S$ is projected onto a line for each $t$ and metric distances are assumed to be preserved in this one dimensional representation. The circles represent samples and the filled circle denotes the sample that is part of the $K$-NN path. The crosses represent the starting point on the first line and the transitions from samples on the $K$-NN path on the subsequent lines. }
\end{wrapfigure}




\paragraph{$K$-Nearest Neighbor Resampling for OPE} Given the data set $\mathcal{D}$ and the feedback policy $u$, the goal of the $K$-nearest neighbor resampling algorithm is to estimate the expected cumulative reward $m(u)$. We achieve this by randomly sampling $K$-NN paths and averaging the cumulative rewards along the sampled $K$-NN paths. This procedure is explicitly formulated in Algorithm \ref{alg:matching}.\footnote{To simplify readability, some algorithmic details are omitted in Algorithm \ref{alg:matching} (e.g., tree-based nearest neighbor search). A more detailed version can be found in the supplementary materials.}

The algorithm produces an estimate of the value by generating trajectories without assuming a parametric model. Instead, we exploit the fact that under sufficient regularity of the underlying stochastic system, state-action pairs that are close in the chosen metric are related to similar transitions and rewards. This also implies that the algorithm does not require optimization. Furthermore, it is directly applicable to problems with continuous state and action spaces and no detailed knowledge of the sampling policy is required. In particular, it can handle the distribution mismatch between random behavior policies and deterministic target policies and it allows for the evaluation of non-stationary policies. The algorithm also allows, with slight modifications, the estimation of path dependent quantities (see Remark \ref{rem:generalization})

The choice of the hyperparameters $K$ and $l$, the number of resampled paths, allows for flexibility in the bias-variance trade-off of the estimation. Our theoretical results require $K_n\rightarrow\infty$ where $K_n=o(n)$ and that $l_n\rightarrow\infty$ as $n\rightarrow\infty$. As mentioned earlier, for $K_n$, standard regression theory suggests a dimension dependent rate of $\mathcal{O}(n^{\frac{2}{d+2}})$ where the dimension $d\geq3$ to optimally balance variance and bias (Theorem 6.2 in \cite{gyorfi2002distribution}).  In the regression case, this choice of $K_n$ leads to minimax optimal convergence for the class of Lipschitz-continuous functions with a fixed Lipschitz constant (i.e., \cite{gyorfi2002distribution}). In practice, a heuristic choice of $K$ is common (i.e., \cite{lall1996nearest}). Finally, the averaging over resampled paths makes the algorithm behave similarly to Monte Carlo methods. An example for this is the empirical observation in Section \ref{sec:experiments} that the standard deviation of the estimation vanishes approximately with the order $O(n^{-0.5})$, if $l_n$ is of the order $O(n)$\footnote{For rates of order $o(n)$, the standard deviation decays slower in experiments.}. The Monte Carlo-like behavior makes the \ref{alg:matching} algorithm particularly well-suited for problems with inherent stochasticity.
\begin{algorithm}
	\renewcommand{\thealgorithm}{KNNR}
	\caption{ $K$-nearest neighbor resampling for OPE}\label{alg:matching}
	 \hspace*{\algorithmicindent} \textbf{Input:}  Data set $\mathcal{D}$; target policy $u$; nearest neighbors parameter $K$; metric $d$ on $S$; number of resampled trajectories $l$
	\begin{algorithmic}[1]
		\STATE $R\gets 0$
	\FOR{$j$ from $1$ to $l$}
	\STATE{$r\gets 0$}
	\STATE{Sample an initial state $X^{\mathrm{NN}}_{j,0}$ from $\nu_0$}
	\FOR{$s$ from $0$ to $T$ }
	\STATE{Randomly choose $\mathcal{K}_s$ where $\mathcal{K}_s \sim \mathcal{V}(\{1,\dots,K\})$ and $\mathcal{K}^s\gets(\mathcal{K}_0,\dots,\mathcal{K}_s)$}
	\STATE{	Find the $\mathcal{K}_s$-nearest neighbor of $(X^{\mathrm{NN}}_{j,s}, u_s(X^{\mathrm{NN}}_{j,s}))$ in $\mathcal{D}$ under $d$ and denote them by $(X^{\mathcal{K}^s}_{i_s,t_s}, U^{\mathcal{K}^s}_{i_s,t_s})$ with the corresponding reward $R^{\mathcal{K}^s}_{i_s,t_s}$}
	\STATE{Set $r\gets r + R^{\mathcal{K}^s}_{i_s,t_s}$ and $X^{\mathrm{NN}}_{j,s+1}\gets X^{\mathcal{K}^s}_{i_s,t_s+1}$ if $s<T$}
	\ENDFOR
	\STATE{$R\gets r + R$}
	\ENDFOR
	\end{algorithmic}
\hspace*{\algorithmicindent} \textbf{Output:} $\frac{R}{l}$
\end{algorithm}

Relying on nearest neighbor searches imposes some limitations on the algorithm. First, at evaluation, we are required to have the data set stored and the computational cost may be relatively high, depending on the size of the data set. As mentioned before, however, our algorithm does not require optimization and, in contrast to similar existing methods, allows for more efficient evaluation using tree-based nearest neighbor search and almost complete parallelization. Also, nearest neighbor search in high-dimensional spaces suffers from the curse of dimensionality. This problem can be potentially alleviated by using appropriate lower dimensional representations \cite{pari2022surprising} and a suitable choice of metric. Furthermore, we need to assume that our decision problem has a Markov structure. This restriction can be potentially circumvented if a sufficient summary statistic of the paths can be identified. The formulation of the algorithm and the theoretical results require the target policies to be deterministic and the behavior policies stationary. These restrictions might be avoidable in future results. Finally, only weak regularity assumptions on transitions and rewards, and a sufficient coverage by the behavioral policy are required.

\begin{rem}
	On an algorithmic level, a crucial difference of Algorithm  \ref{alg:matching} shown above to the MFMC method proposed by \cite{fonteneau2010model,fonteneau2013batch} occurs in line 8. At this point of the algorithm, MFMC would remove the used sample from the data set while \ref{alg:matching} keeps the data set unchanged. This seemingly small difference leads to the favourable computational properties of the \ref{alg:matching} compared to the MFMC as specified in the introduction. 
\end{rem}

%% file: 20231121theory_main.tex
\section{Consistency of $K$-nearest neighbor resampling}\label{sec:theory_main}
In this section, we introduce our theoretical results on the consistency of the estimator implemented by Algorithm \ref{alg:matching}. This means we can guarantee under certain weak assumptions in the setting previously introduced that the mean-squared error related to the off-policy value estimation by $K$-nearest neighbor resampling goes to zero as the number of sampled episodes $n$ increases. Note that this is in particular equivalent to the bias and variance of value vanishing as more data becomes available. To our knowledge,this is the first consistency result under episodic sampling for non-deterministic environments with continuous state-action spaces.

We will first provide the assumptions needed and introduce a regression notation for the algorithm's output, before stating our main result and 
providing the idea of the proof. A full proof is given in Section \ref{sec:proof} with some detailed technical steps in the supplementary materials.
 
\begin{ass}\label{ass:absolute_cont_mult_main}\leavevmode
	\begin{enumerate}
		\item The probability of a tie occurring in any nearest neighbor search is zero.
		\item 	Let $(\pi_t)_{t=0}^{T}$ be a family of probability measures where $\pi_t$  is the measure of $(X_t,U_t)$ and let $\pi_t|_x$ denote the marginal distribution in the first component. Assume that for all $t$ and for all $x\in supp(\pi_t|_x)$  there exists a $\tilde{t}<T-1$, such that $(x,u_t(x))\in supp(\pi_{\tilde{t}})$.
	\end{enumerate}	
\end{ass} 
The first assumption is common in the analysis of nearest neighbor-based methods (e.g., \cite{gyorfi2002distribution}, Chapter 6) and can usually be replaced by introducing a tie-breaking mechanism. We omit treating this case to simplify argumentation and note that this assumption is satisfied for any $K$-NN search where the matching variables have non-atomic probability distributions. The second condition guarantees sufficient exploration of the behaviour policy and is comparable with common coverage assumptions (cf. \cite{uehara2021pessimistic} for an overview).

To state the main result, we will introduce the K-NN path resampling estimator (KNNR) for the value function that corresponds to the estimator produced by Algorithm \ref{alg:matching}.  With the notation introduced in the previous section, $R_{\tau}^{k^{\tau}}(x;u)$ denotes the reward in period $\tau$ of the $K_n$ nearest neighbor path $k^{\tau}$ starting at the state $x\in\mathcal{X}$ under the target policy $u$. Draw independently $l\in\mathbb{N}$ sequences $\mathcal{K}^T_j=(\mathcal{K}_{0,j},\dots, \mathcal{K}_{T,j})$ of length $T+1$, where $\mathcal{K}^T_j$ is a sequence of iid random variables distributed uniformly on $\{1,\ldots, K_n\}$ and $j\leq l$ with $j\in\mathbb{N}$. Then, the resampling regression is given by $\msub_n(x,u,\tau,j) =  \sum_{k^{\tau}}^{K_n} R_{\tau}^{k^{\tau}}(x;u)I_{k^{\tau}=\mathcal{K}^{\tau}_j}$, where $\sum_{k^{\tau}}^{K_n} =\sum_{k_0=1}^{K_n}\cdots\sum_{k_{\tau}=1}^{K_n}$ and the resampling estimator for $\mathbb{E}[R_{\tau}|_u]$ is $\msub_{l,n}(u,\tau) = \frac{1}{l}\sum_{j=1}^l \msub_n(\tilde{X}_j,u,\tau,j)$, where $\tilde{X}_j$ is sampled independently from the initial state distribution. Then, the KNNR estimator for the value of policy $u$ is given by
\begin{equation}
	\msub_{l,n}(u) = \sum_{\tau=0}^{T} \msub_{l,n}(u,\tau).
\end{equation}
With this definition, we can state that estimator calculated by \ref{alg:matching} is consistent:
\begin{thm}\label{thm:resampling_main}
	Given the setting described in Section \ref{sec:algo} and under Assumption \ref{ass:absolute_cont_mult_main},  $\msub_{l,n}(u)$ is consistent at each family of feedback policies $u=(u_t)_{t=0}^T$ with $u_t:\mathcal{X}\rightarrow\mathcal{U}$, i.e.,
	\begin{equation}
		\lim_{n\rightarrow\infty}\mathbb{E}\left[\left(\msub_{l_n,n}(u)-m(u)\right)^2\right]=0
	\end{equation}
	with $l_n\rightarrow\infty$, $K_n\rightarrow\infty$ and $\frac{K_n}{n}\rightarrow0$.
\end{thm}
The proof follows in three steps. We will briefly describe each of these steps, but will refer the details to Section \ref{sec:proof}.

\subsection{Local Averaging Regression Estimators} The first step is to show that a more general class of estimators based on local-averaging regressions is consistent under certain conditions on the local weights. Denote $D = \{(X_i,U_i)\}_{i=1}^n$. Define the local averaging regression for $R_{\tau}|_u$ for the initial condition $x$ and policy $u$ by $\widehat{m}_n(x,u,\tau) =\sum_{i^{\tau}, t^{\tau}}^{n, T} W_{n,i^{\tau}, t^{\tau}}(x,u) R_{i_{\tau},t_{\tau}}$. The weights are for each $(i^{\tau}, t^{\tau})$ real-valued functions of $x$, $u$ and all the state action pairs $D$ (i.e., $W_{n,i^{\tau}, t^{\tau}}(x,u) = W_{n,i^{\tau}, t^{\tau}}(x,u,D)\in\mathbb{R}$). Note that for the local-averaging regression we sum $\tau$-times over the rewards of the entire data set and assign each summand a unique weight. The estimator for $\mathbb{E}[R_{\tau}|_u]$ is given by $\widehat{m}_{l,n}(u,\tau) = \frac{1}{l}\sum_{j=1}^l \widehat{m}_n(\tilde{X}_{j},u,\tau)$, where $\tilde{X}_j$ is sampled independently from the initial state distribution. The local-averaging regression estimator for  $m(u)$ is $\widehat{m}_{l,n}(u) = \sum_{\tau=0}^{T} \widehat{m}_{l,n}(u,\tau)$.

We define $\{(\tilde{X}_{i,t+s},\tilde{U}_{i,t+s})\}_{s=1}^{T-t}$ which is generated by the iid sequence of random functions $(\tilde{\alpha}_{i,t+s})_{s=0}^{T-t-1}$ distributed as in the original data set generation but independent of the original data set and starting from the pair $(X_{i,t},U_{i,t})$, i.e. $(\tilde{X}_{i,t+s},\tilde{U}_{i,t+s})=\tilde{\alpha}_{i,t+s-1}\tilde{\alpha}_{i,t+s-2}\cdots\tilde{\alpha}_{i,t}(X_{i,t},U_{i,t})$. Define $D_{i,t}= (D \setminus\{(X_{i,t+s},U_{i,t+s})\}_{s=1}^{T-t}\cup\{(\tilde{X}_{i,t+s},\tilde{U}_{i,t+s})\}_{s=1}^{T-t}$ as the data set where $\{(X_{i,t+s},U_{i,t+s})\}_{s=1}^{T-t}$ is replaced by $\{(\tilde{X}_{i,t+s},\tilde{U}_{i,t+s})\}_{s=1}^{T-t}$. Analogously, for any multi-index $i^{l},t^{l}$, we recursively define $D_{i^{k:l},t^{k:l}}=(D_{i^{k+1:l},t^{k+1:l}}\setminus\{(X_{i_k,t_k+s},U_{i_k,t_k+s})\}_{s=1}^{T-t_k})\cup\{(\tilde{X}_{i_k,t_k+s},\tilde{U}_{i_k,t_k+s})\}_{s=1}^{T-t_k}$. For a weight, we highlight the data set it depends on by writing $W_{n,i^{\tau}, t^{\tau}}(\tilde{D})$ where $\tilde{D}$ is a data set. With this notation we can now state the conditions on the weights of the local averaging regression under which the estimation is consistent:

\begin{ass}\label{ass:weights_mult}
	For all $\tau=0,...,T$ and with $W_{n,i^{\tau}, t^{\tau}}(\tilde{X}_{1},u)$ denoted by $W_{n,i^{\tau}, t^{\tau}}$, the following holds:
	\begin{enumerate}[leftmargin=*]
		\item For either $\tilde{D}_1=\tilde{D}_2=D$ or $\tilde{D}_1=D_{i^{\tau}, t^{\tau}}$, $\tilde{D}_2=D_{j^{\tau}, s^{\tau}}$,
		\begin{equation*}
			\lim_{n\rightarrow\infty}\mathbb{E}\Bigg[\sum_{i^\tau,t^\tau}^{n,T}\sum_{\substack{j^\tau,s^\tau\\  |[i^\tau,j^\tau]|<2\tau+2}}^{n,T} |W_{n,i^{\tau}, t^{\tau}}(\tilde{D}_1) W_{n,j^\tau,s^\tau}(\tilde{D}_2)| \Bigg] =0
		\end{equation*}
		
		\item There is a $L\geq 1$ such that 
		\begin{equation*}
			\mathbb{P}\Bigg( \sum_{i^{\tau}, t^{\tau}}^{n, T} | W_{n,i^{\tau}, t^{\tau}}|\leq L\Bigg) = 1
		\end{equation*}
		\item For all $a>0$, 
		\begin{align*}
			\begin{split}
				&\lim_{n\rightarrow\infty}\mathbb{E}\Bigg[\sum_{i^{\tau}, t^{\tau}}^{n, T}| W_{n,i^{\tau}, t^{\tau}}|\Bigg(
				I_{d\left((X_{i_{0},t_{0}},U_{i_{0},t_{0}}),(\tilde{X}_{1}, u_{0}(\tilde{X}_{1}))\right)>a}+\sum_{k=0}^{\tau} I_{\tau-k}(a)\Bigg)\Bigg]=0
			\end{split}
		\end{align*}
		where
		\begin{equation*}
			I_{m}(a)  = 
			I_{d\left((X_{i_{m},t_{m}},U_{i_{m},t_{m}} ),(X_{i_{m-1},t_{m-1}+1}, u_{m}(X_{i_{m-1},t_{m-1}+1} ))\right)>a}
		\end{equation*}
		and $X_{i_{-1},t_{-1}}=\tilde{X}_{1}$ for all $i^{\tau}, t^{\tau}$.

		\item It holds in probability that
		\begin{equation*}
			\sum_{i^{\tau}, t^{\tau}}^{n, T} W_{n,i^{\tau}, t^{\tau}}\rightarrow1
		\end{equation*}
		
		\item
		\begin{equation*}
			\lim\limits_{n\rightarrow\infty}\mathbb{E}\Bigg[\Bigg(\sum_{i^{\tau}, t^{\tau}}^{n, T} |W_{n,i^{\tau}, t^{\tau}}(D_{i^{\tau},t^{\tau}})-W_{n,i^{\tau}, t^{\tau}}(D)|\Bigg)^2\Bigg]= 0
		\end{equation*}
		and
		\begin{align*}
			\lim\limits_{n\rightarrow\infty}\mathbb{E}\Bigg[\sum_{i^{\tau}, t^{\tau}}^{n, T}& \sum_{\substack{j^\tau,s^\tau\\  |[i^\tau,j^\tau]|=2\tau+2}}^{n,T}
			|W_{n,i^{\tau}, t^{\tau}}(D_{i^{\tau},t^{\tau}})|\\
			&\cdot|W_{n,j^{\tau}, s^{\tau}}(D_{j^{\tau},s^{\tau}})
			-W_{n,j^{\tau}, s^{\tau}}(D_{(i^{\tau},j^{\tau}),(t^{\tau},s^{\tau})})|\Bigg]= 0
		\end{align*}
		
	\end{enumerate}
\end{ass}

\begin{thm}\label{thm:stone_mult}
	Given the setting described in Section 2 and under Assumption \ref{ass:weights_mult}, $\hat{m}_{l,n}(u)$  is consistent at each family of feedback policies $u=(u_t)_{t=0}^T$ with $u_t:\mathcal{X}\rightarrow\mathcal{U}$, i.e.,
	\begin{equation}
		\lim_{n\rightarrow\infty}\mathbb{E}\left[\left(\hat{m}_{l_n,n}(u)-m(u)\right)^2\right]=0
	\end{equation}
	where $l_n\rightarrow\infty$.
\end{thm}
The proof of Theorem \ref{thm:stone_mult} is given in Section \ref{subsec:stone}.

\begin{rem}
Theorem \ref{thm:stone_mult} gives consistency of the local averaging regression estimator under certain conditions on the weights. This theorem can be viewed as a generalization of Stone's Theorem (see \cite{gyorfi2002distribution}, Theorem 4.1). In particular, Theorem \ref{thm:stone_mult} generalizes Stone's Theorem in three important aspects: First, Theorem \ref{thm:stone_mult} allows an episodic sampling setting where within each episode the data is not independent and not identically distributed. Second, the counterfactual estimation problem arising through off-policy evaluation is taken into account and, third, we allow for a sequential dependence structure in the weights and through this in the regression. As the conditions on the weights that we propose are still general, we believe that this result can be of independent interest and potentially be applied to a wide range of local averaging regression estimators such as partitioning or kernel-based methods.
\end{rem}

\subsection{The $K_n$-NN path regression estimator}The second step of the proof of Theorem \ref{thm:resampling_main} is to show that the conditions on the weights in Theorem \ref{thm:stone_mult} are satisfied by a regression estimator based on K-NN paths. Define the set $\Lambda_n(x,u)=\{K_n\text{-NN paths starting at } x \text{ under } u\}$ and introduce the $K_n$-NN path regression by $\mnn_n(x,u,\tau) = \sum_{i^{\tau}, t^{\tau}}^{n, T}\frac{1}{(K_n)^\tau} I_{(i^{\tau}, t^{\tau}) \in \Lambda_n(x,u)}\cdot R_{i_{\tau},t_{\tau}}$ or equivalently by $\mnn_n(x,u,\tau) = \frac{1}{(K_n)^\tau} \sum_{k^{\tau}}^{K_n} R_{\tau}^{k^{\tau}}(x;u)$. For $\mathbb{E}[R_{\tau}|_u]$, we use $\mnn_{l,n}(u,\tau) = \frac{1}{l}\sum_{j=1}^l \mnn_n(\tilde{X}_{j},u,\tau)$ where $\tilde{X}_j$ is sampled independently from the initial state distribution and, for $m(u)$ the estimator is $\mnn_{l,n}(u) =\sum_{\tau=0}^{T-1} \mnn_{l,n}(u,\tau)$.  We identify the weights in the $K_n$-NN path regression by $V_{n,i^{\tau}, t^{\tau}}\coloneqq (1/(K_n)^{\tau+1}) I_{(i^{\tau}, t^{\tau}) \text{ is a } K_n\text{-NN path}}(x,u)$.

\begin{thm}\label{thm:NN_consistent_mult_main}
	Given the setting described in Section \ref{sec:algo} and under Assumption \ref{ass:absolute_cont_mult_main},  it holds at each family of feedback policies $u=(u_t)_{t=0}^T$ with $u_t:\mathcal{X}\rightarrow\mathcal{U}$,  that
	\begin{equation}
		\lim_{n\rightarrow\infty}\mathbb{E}\left[\left(\mnn_{l_n,n}(u)-m(u)\right)^2\right]=0
	\end{equation}
	with $l_n\rightarrow\infty$, $K_n\rightarrow\infty$ and $\frac{K_n}{n}\rightarrow0$.
\end{thm}
This theorem is proved by showing that the weights $V_{n,i^{\tau}, t^{\tau}}$ satisfy the conditions of Theorem \ref{thm:stone_mult}. This is done by first analysing the asymptotic behaviour of $\mathbb{E}[V_{n,i^{\tau}, t^{\tau}}]$. This analysis is key to showing that a nearest neighbour-based OPE algorithm can directly use episodic sampling. With this result, we investigate the limit of $\sum_{i^{\tau}, t^{\tau}}^{n,T}V_{n,i^{\tau}, t^{\tau}}$. Finally, it remains to show that with Assumption  \ref{ass:absolute_cont_mult_main} at any relevant point in $\mathcal{S}$ the distance to the $K_n$-nearest neighbor vanishes. The proof of Theorem \ref{thm:NN_consistent_mult_main} is given in Section \ref{subsec:NN_consistent_mult_main}.

As the third and final step, one identifies $\msub_{l,n}(u)$ as a subsampled version of the estimator  $\mnn_{l,n}(u)$ in the sense that $\msub_{l,n}(u)$ averages over a growing number of rewards chosen along randomly selected $K_n$-paths that are also contained in $\mnn_{l,n}(u)$. Then showing that the mean squared error between $\msub_{l,n}(u)$ and $\mnn_{l,n}(u)$ vanishes yields the statement. This proof is given in Section \ref{subsec:resampling_main}.

%% file: 20230506experiments_main.tex
\section{Experiments}\label{sec:experiments}
We demonstrate the effectiveness of the \ref{alg:matching} algorithm for off-policy evaluation on three stochastic control environments with varying control mechanisms, areas of applications, sources of uncertainty and dimensionality. Namely, we study a linear quadratic regulator environment, trade execution in limit order books and online stochastic bin packing. As baselines, we consider weighted per episode importance sampling (PEIS) \cite{precup2000eligibility},  weighted per decision importance sampling (PDIS) \cite{precup2000eligibility}, weighted doubly-robust (WDR) \cite{thomas2016data}, model-free Monte Carlo (MFMC) \cite{fonteneau2010model,fonteneau2013batch}, fitted Q-evaluation with a linear-quadratic regression (FQE Lin) \cite{le2019batch}, fitted Q-evaluation with a nearest neighbor regression (FQE NN) \cite{le2019batch} and the naive average of rewards generated by the behaviour policy (NA). Before presenting the results, we shortly highlight the differentiating properties of the test environments. Note that more details on the test environments, the target and behaviour policies, the baselines, results and a link to the implementation are included in the supplementary materials.

\begin{figure}
	\centering
	\includegraphics[width=\textwidth]{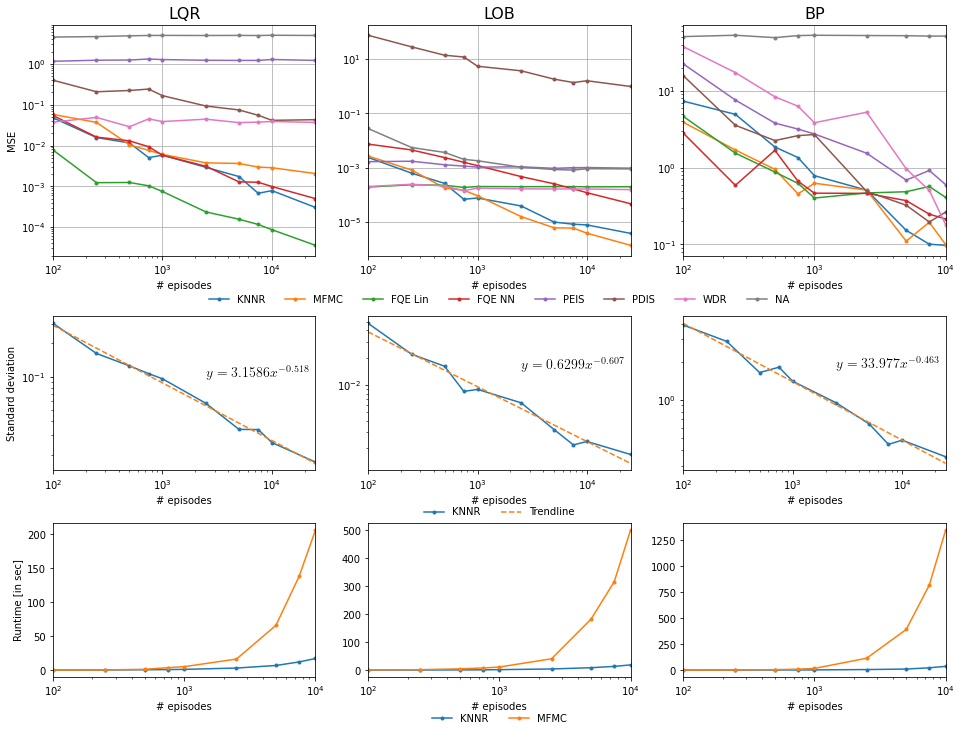}
	\caption{Experimental results: Each column corresponds to one test environment. First row: Median MSEs of the value estimates (where the target for the value is computed via Monte Carlo roll outs using the target policy) versus the number of episodes sampled under the behaviour policy (log-log plot). Second row: Standard deviation in the resampling estimates of the value as a function of the number of episodes in the data set. Third row: Comparison of median runtimes for the resampling algorithm and MFMC depending on the number of episodes in the data set.}
	\label{fig:results}
\end{figure}

\paragraph{Linear Quadratic Regulator} The linear quadratic regulator problem (LQR) is often seen as the simplest non-trivial example for stochastic control and was proposed as an important test-bed in reinforcement learning by \cite{recht2019tour}. The LQR problem has a continuous state and action space, with the dynamics  linear in the state and action, and the reward quadratic in the state and action. It can be viewed as a local approximation of more general non-linear stochastic control problems. For this experiment, we consider the double integrator example from \cite{recht2019tour} with a higher variance to highlight the benefits of the resampling approach in stochastic settings. The LQR environment has a two dimensional state space and a one dimensional action space. 

\paragraph{Optimal Execution in Limit Order Book Markets} The second environment models optimal execution in limit order books (LOB), as in \cite{cartea2015algorithmic}, Chapter 8.2. In this scenario, a trading agent has the task to liquidate a position of assets in a fixed amount of time, trying to maximise trade revenues. They do so by quoting a price difference to a mid-price that varies exogenously. 
A larger difference to the mid-price is associated with a lower probability of a sale but a higher revenue, and vice versa. The state consists of the current price of the asset (continuous) and inventory (finite). The action space is one dimensional and continuous. The target policy is non-stationary, depending on time and the current level of inventory. Similar control mechanisms are central in the dynamic pricing literature (e.g., \cite{gallego1994optimal}).

\paragraph{Stochastic Online Bin Packing}
We consider a slightly altered version of the stochastic online bin packing (BP) environment as introduced in \cite{gupta2012online,balaji2019orl}. The general idea of the bin packing problem is that over time items of random size arrive and need to be allocated to bins of fixed size, with the goal to minimize open capacities at each point in time. Bin packing is used to model many different resource allocation problems in both operations research and computer science. By our modification, we allow the size of the next arriving item to be negatively correlated with current free capacity. In operations research, this dependence could originate from having a separate dynamic pricing policy depending on available capacities \cite{gallego1994optimal,bayliss2019dynamic}. This problem differs from the previous problems in that both the state and action space are finite. The state space consists of the number of bins at each potential fill level and the size of the arriving item. The action space contains the possible assignments of the arriving item. We choose this problem as a benchmark as it has a higher dimensional state space ($d=10$) and highlights the wider applicability of the resampling algorithm to non-standard stochastic control environments.  

\paragraph{Results} We report our results in Figure \ref{fig:results}.  From the first row, one can see that in all environments our resampling algorithm exhibits decreasing 
Mean-Square Errors (MSEs) for increasing data set sizes, measured in the number of episodes. The performance of the resampling algorithm versus the baselines is competitive. It consistently outperforms PEIS, PDIS and naive averages, and is mostly better or performs similarly to the other baselines. The exception is FQE Lin in the LQR environment. This can be explained by the fact that the value function in the LQR problem has exactly the same functional form as the linear quadratic regression used in FQE Lin for value function estimation. The second row highlights that the standard deviation of the resampling estimates vanishes with approximately the rate $O(n^{-0.5})$ as for Monte Carlo methods. Finally, the third row shows that the actual runtime for the resampling algorithm is lower compared to MFMC and reaches an order of magnitude difference for medium sized data sets.

%% file: 20231124proofs.tex
\section{Proofs}
\label{sec:proof}

In this section, we give the proofs for the main results. Other additional supporting results are included in the appendix.
\subsection{Proof of Theorem \ref{thm:stone_mult}}\label{subsec:stone}
\begin{proof}[Proof of Theorem \ref{thm:stone_mult}]
Note that 
	\begin{align*}
		\mathbb{E}\Bigg[\Bigg(\sum_{\tau=0}^{T}\hat{m}_{l_n,n}(u,\tau)-\mathbb{E}[(R_{\tau}|_u)]\Bigg)^2\Bigg]
		\leq (T+1)\sum_{\tau=0}^{T}\mathbb{E}\left[\left(\hat{m}_{l_n,n}(u,\tau)-\mathbb{E}[(R_{\tau}|_u)]\right)^2\right].
	\end{align*}
	
	It will be sufficient to show that
	\begin{equation*}
		\lim_{n\rightarrow\infty}\mathbb{E}\left[\left(\hat{m}_{l_n,n}(u,\tau)-\mathbb{E}[(R_{\tau}|_u)]\right)^2\right]=0
	\end{equation*}
	for a fixed $\tau$ since the arguments are analogous for all $\tau$.

	For that purpose, consider
	\begin{align*}
		\mathbb{E}&\left[\left(\hat{m}_{l_n,n}(u,\tau)-\mathbb{E}[(R_{\tau}|_u)]\right)^2\right]\\
		&\quad\leq  2\mathbb{E}\left[\left(\frac{1}{{l_n}}\left(\sum_{j=1}^{l_n}\left(\sum_{i^{\tau}, t^{\tau}}^{n, T} W_{n,i^{\tau}, t^{\tau}}(\tilde{X}_{j},u) r_{i_{\tau},t_{\tau}}\right)-\mathbb{E}[R_{\tau}|_u| \tilde{X}_{j}]\right)\right)^2 \right]\\
		&\quad+ 2\mathbb{E}\left[\left(\frac{1}{{l_n}}\sum_{j=1}^{l_n}\mathbb{E}[R_{\tau}|_u| \tilde{X}_{j}]-m(u)\right)^2 \right]\\
		\eqqcolon& 2C_n + 2D_n.
	\end{align*}

	By Jensen's inequality, the fact that the sampling of the initial state is iid and independent of the data set, we get
	\begin{equation*}
		C_n \leq \mathbb{E}\Bigg[\Bigg(\sum_{i^{\tau}, t^{\tau}}^{n, T} W_{n,i^{\tau}, t^{\tau}}(\tilde{X}_{1},u) R_{i_{\tau},t_{\tau}}-\mathbb{E}[R_{\tau}|_u| \tilde{X}_{1}]\Bigg)^2 \Bigg] \eqqcolon C'_n.
	\end{equation*}
	Further\footnote{For readability, we omit $(\tilde{X}_{1},u)$ from the weights.}, 
	\begin{align*}
		C'_n \leq & 3\mathbb{E}\Bigg[\Bigg(\sum_{i^{\tau}, t^{\tau}}^{n, T} W_{n,i^{\tau}, t^{\tau}}(R_{i_{\tau},t_{\tau}}-r(X_{i_{\tau},t_{\tau}},U_{i_{\tau},t_{\tau}}))\Bigg)^2 \Bigg]\\
		& + 3\mathbb{E}\Bigg[\Bigg(\sum_{i^{\tau}, t^{\tau}}^{n, T} W_{n,i^{\tau}, t^{\tau}}\Big(r(X_{i_{\tau},t_{\tau}},U_{i_{\tau},t_{\tau}})\\
		&\qquad\qquad-r(\alpha_{i_{\tau-1},t_{\tau-1}}|_u\circ\cdots\circ\alpha_{i_{0},t_{0}}|_u(\tilde{X}_{1},u_0(\tilde{X}_{1})))\Big)\Bigg)^2 \Bigg]\\
		&+ 3 \mathbb{E}\Bigg[\Bigg(\sum_{i^{\tau}, t^{\tau}}^{n, T} W_{n,i^{\tau}, t^{\tau}}r(\alpha_{i_{\tau-1},t_{\tau-1}}|_u\circ\cdots\circ\alpha_{i_{0},t_{0}}|_u(\tilde{X}_{1},u_0(\tilde{X}_{1})))\\
		& \qquad-\mathbb{E}[R_{\tau}|_u| \tilde{X}_{1}]\Bigg)^2 \Bigg]\\
		\eqqcolon& 3E_n + 3F_n + 3G_n   
	\end{align*}
	In Proposition \ref{prop:stone_mult} in the supplementary materials we show that $D_n$, $E_n$, $F_n$, and $G_n$ vanish using Assumption \ref{ass:weights_mult}.
\end{proof}

\subsection{Proof of Theorem \ref{thm:NN_consistent_mult_main}}\label{subsec:NN_consistent_mult_main}
Due to the recursive definition of $K_n$-NN paths, $V_{n,i^{\tau}, t^{\tau}}$ can be written as a product of weights:
\begin{equation*}
	V_{n,i^{\tau}, t^{\tau}}(\tilde{D})=V_{n,i_{0}, t_{0}}(\tilde{D})\prod_{k=1}^{\tau} V_{n,[(i_{k-1}, t_{k-1}),(i_{k}, t_{k})]}(\tilde{D})
\end{equation*}
where
\begin{equation*}
	V_{n,i_0, t_0}(\tilde{X}_1,u,\tilde{D})=\frac{1}{K_n}I_{(X_{i_0, t_0},U_{i_0, t_0}) \text{ is a } K_n-\text{NN}\text{ in }\tilde{D}}(\tilde{X}_1,u_0(\tilde{X}_1))
\end{equation*}
and
\begin{align*}
	&V_{n,[(i_{k-1}, t_{k-1}),(i_{k}, t_{k})]}(\tilde{D})=\\
	&\quad\frac{1}{K_n}I_{(X_{i_{k}, t_{k}},U_{i_{k}, t_{k}}) \text{ is a } K_n-\text{NN}\text{ in }\tilde{D}}(X_{i_{k-1}, t_{k-1}+1},u_{t+1}(X_{i_{k-1}, t_{k-1}+1})).
\end{align*} 
We omit the dependence on the data set $\tilde{D}$, if the unaltered data $D$ is used. Finally, note that by definition it holds almost surely for $k=1,\dots, \tau$ that
\begin{equation*}
	\sum_{i=1}^{n}\sum_{t=0}^{T-1}V_{n,[(i_{k-1}, t_{k-1}),(i, t)]}=1\quad\text{and}\quad 	\sum_{i=1}^{n}\sum_{t=0}^{T-1}V_{n,i, t}(\tilde{X}_1,u)=1.
\end{equation*}
The next two propositions analyse the asymptotic behaviour of $\mathbb{E}[V_{n,i^{\tau}, t^{\tau}}]$ and  $\sum_{i^{\tau}, t^{\tau}}^{n,T}V_{n,i^{\tau}, t^{\tau}}$, respectively.

\begin{prop}\label{prop:asymptotic}
	Let Assumption 1.1 be satisfied. For all $\tau\leq T$ and $(i^{\tau},t^{\tau})$, it holds for all $n$ that
	\begin{equation}
		\mathbb{E}[V_{n,i^{\tau}, t^{\tau}}]\leq (n-\tau-1)^{-|i^{\tau}|}K^{-(\tau+1-|i^{\tau}|)}_n.
	\end{equation}
\end{prop}

\begin{proof}
	We prove this statement by induction over $\tau$.
	First, consider $\tau=0$:
	\begin{align*}
		\mathbb{E}[V_{n,i_0, t_0}(\tilde{X}_1,u)]=& \frac{1}{K_n}\mathbb{P}((X_{i_0, t_0},U_{i_0, t_0}) \text{ is a } K_n-\text{NN}\text{ of }(\tilde{X}_1,u_0(\tilde{X}_1)) \text{ in }D )\\
		\leq &\frac{1}{K_n} \mathbb{P}((X_{i_0, t_0},U_{i_0, t_0}) \text{ is a } K_n-\text{NN}\text{ of }(\tilde{X}_1,u_0(\tilde{X}_1)) \text{ in }D_{t_0} )\\
		\leq& \frac{K_n}{nK_n }\leq \frac{1}{n-1}
	\end{align*}
	where $D_{t_0}\coloneqq\{(X_{i, t_0},U_{i, t_0})\}_{i=1}^n$ and we use that the samples in $D_{t_0}$ are iid.
	
	For $\tau\Rightarrow\tau+1$ where $\tau+1\leq T$, we consider two cases. In the first case we assume that $|i^{\tau}|=|i^{\tau+1}|$:
	\begin{align*}
		\mathbb{E}&[V_{n,i^{\tau+1}, t^{\tau+1}}]=\mathbb{E}[V_{n,[(i_{\tau}, t_{\tau}),(i_{\tau+1}, t_{\tau+1})]}V_{n,i^{\tau}, t^{\tau}}]\\
		=&\mathbb{E}[V_{n,[(i_{\tau}, t_{\tau}),(i_{\tau+1}, t_{\tau+1})]}\mid V_{n,i^{\tau}, t^{\tau}}=K_n^{-(\tau+1)}]\mathbb{E}[V_{n,i^{\tau}, t^{\tau}}]\\
		\leq& (n-1)^{-|i^{\tau}|}K^{-(\tau+2-|i^{\tau}|)}_n\\
		&\cdot \mathbb{P}\Big((X_{i_{k}, t_{k}},U_{i_{k}, t_{k}}) \text{ is a }
		 K_n-\text{NN}\text{ of }(X_{i_{k-1}, t_{k-1}+1},u_{t+1}(X_{i_{k-1}, t_{k-1}+1}))\text{ in }D, (i^{\tau},t^{\tau})\\
		 &\qquad\qquad\mid V_{n,i^{\tau}, t^{\tau}}=K_n^{-(\tau+1)}\Big)\\
		\leq&	(n-1)^{-|i^{\tau+1}|}K^{-(\tau+2-|i^{\tau+1}|)}_n
	\end{align*}                         
	For the second case consider $|i^{\tau}|+1=|i^{\tau+1}|$. Assume that 
	\begin{equation*}
		\mathbb{E}[V_{n,i^{\tau+1}, t^{\tau+1}}]> (n-\tau-2)^{-|i^{\tau+1}|}K^{-(\tau+2-|i^{\tau+1}|)}_n.
	\end{equation*}
	By the iid assumption for the sampling of the different episodes, it has to hold that
	\begin{equation}\label{eq:contradict}
		\mathbb{E}[V_{n,[(i_{\tau}, t_{\tau}),(j, t_{\tau+1})]}V_{n,i^{\tau}, t^{\tau}}]> (n-\tau-2)^{-|i^{\tau+1}|}K^{-(\tau+2-|i^{\tau+1}|)}_n
	\end{equation}
	for all $j$ where $|(i^{\tau},j)|>|i^{\tau}|$.
	By induction hypothesis it follows that
	\begin{align*}
		\mathbb{E}[V_{n,[(i_{\tau}, t_{\tau}),(j, t_{\tau+1})]}\mid V_{n,i^{\tau}, t^{\tau}}=1/K^{\tau+1}_n]> (n-\tau-2)^{-1}
	\end{align*}
	Recall that $\sum_{i=1}^{n}\sum_{t=0}^{T-1}V_{n,[(i_{\tau}, t_{\tau}),(i, t)]}=1$ almost surely and, thus,
	\begin{equation*}
		\sum_{i=1}^{n}\sum_{t=0}^{T-1}\mathbb{E}[	\mathbb{E}[V_{n,[(i_{\tau}, t_{\tau}),(i, t)]}\mid V_{n,i^{\tau}, t^{\tau}}=1/K^{\tau+1}_n]]=1
	\end{equation*} 
	We use the notation $i\in i^{\tau}$ as shorthand for $i\in {i_k}_{k=0}^{\tau}$ and $i\notin i^{\tau}$ for $i\notin {i_k}_{k=0}^{\tau}$.  Splitting the sum and using condition (\ref{eq:contradict}), yields
	\begin{align*}
		1=& \sum_{\substack{i=1,\\ i\in i^{\tau}}}^{n}\sum_{t=0}^{T-1}	\mathbb{E}[V_{n,[(i_{\tau}, t_{\tau}),(i, t)]}\mid V_{n,i^{\tau}, t^{\tau}}=1/K^{\tau+1}_n] \\
		&+ \sum_{\substack{i=1,\\ i\notin i^{\tau}}}^{n}\sum_{\substack{t=0,\\ t\neq t_{\tau+1}}}^{T-1}	\mathbb{E}[V_{n,[(i_{\tau}, t_{\tau}),(i, t)]}\mid V_{n,i^{\tau}, t^{\tau}}=1/K^{\tau+1}_n]\\
		&+\sum_{\substack{i=1,\\ i\notin i^{\tau}}}^{n}	\mathbb{E}[V_{n,[(i_{\tau}, t_{\tau}),(i, t_{\tau+1})]}\mid V_{n,i^{\tau}, t^{\tau}}=1/K^{\tau+1}_n]\\
		\geq& \sum_{\substack{i=1,\\ i\notin i^{\tau}}}^{n}	\mathbb{E}[V_{n,[(i_{\tau}, t_{\tau}),(i, t_{\tau+1})]}\mid V_{n,i^{\tau}, t^{\tau}}=1/K^{\tau+1}_n]\\
		>& \frac{n-\tau-2}{n-\tau-2}=1
	\end{align*}
	This is a contradiction. Hence, the induction step is satisfied and the proposition holds. 
\end{proof}

\begin{prop}\label{prop:vanishing_weights}
	Let Assumption 1.1 be satisfied. Then it holds that:
	\begin{enumerate}[leftmargin=*]
		\item 	
		\begin{equation}
			\lim_{n\rightarrow\infty}\mathbb{E}\Bigg[\sum_{i^\tau,t^\tau}^{n,T}\sum_{\substack{j^\tau,s^\tau\\  |[i^\tau,j^\tau]|<2\tau+2}}^{n,T} V_{n,i^{\tau}, t^{\tau}}(\tilde{D}_1) V_{n,j^\tau,s^\tau}(\tilde{D}_2) \Bigg] =0
		\end{equation}
		for $\tilde{D}_1=\tilde{D}_2=D$ and $\tilde{D}_1=D_{i^{\tau}, t^{\tau}}$, $\tilde{D}_2=D_{j^{\tau}, s^{\tau}}$.
		\item 
		\begin{equation}
			\lim\limits_{n\rightarrow\infty}\mathbb{E}\Bigg[\Bigg(\sum_{i^{\tau}, t^{\tau}}^{n, T} |V_{n,i^{\tau}, t^{\tau}}(D_{i^{\tau},t^{\tau}})-V_{n,i^{\tau}, t^{\tau}}(D)|\Bigg)^2\Bigg]= 0.
		\end{equation}
		\item 
		\begin{equation}
			\begin{split}
				\lim\limits_{n\rightarrow\infty}\mathbb{E}&\Bigg[\sum_{i^{\tau}, t^{\tau}}^{n, T} \sum_{\substack{j^\tau,s^\tau\\  |[i^\tau,j^\tau]|=2\tau+2}}^{n,T}
				V_{n,i^{\tau}, t^{\tau}}(D_{i^{\tau},t^{\tau}})\\
				&\qquad\cdot|V_{n,j^{\tau}, s^{\tau}}(D_{j^{\tau},s^{\tau}})
				-V_{n,j^{\tau}, s^{\tau}}(D_{(i^{\tau},j^{\tau}),(t^{\tau},s^{\tau})})|\Bigg]= 0
			\end{split}
		\end{equation}
	\end{enumerate}
\end{prop}
\begin{proof}
	The proof relies on different variations of Proposition \ref{prop:asymptotic}. The details can be found in the appendix.
\end{proof}
The following Lemma guarantees that under Assumption  \ref{ass:absolute_cont_mult_main} at any relevant point in $\mathcal{S}$ the distance to the $K_n$-nearest neighbor vanishes.

\begin{lem}\label{lem:vanish_dist}
	Under Assumption \ref{ass:absolute_cont_mult_main} , let $\lim_{n\rightarrow\infty}K_n/n = 0$ and $x_0\in supp(\pi_t|_x)$ for some $t=0,\dots,T-1$, then
	\begin{equation}\label{eq:vanishing_distance_0}
		d\left( (x_0, u_0(x_0)),(X_{(K_n,n)},U_{(K_n,n)})(x_0, u_0(x_0))\right)\rightarrow0
	\end{equation}
	and, for all $\tau=1,\dots,T-1$ and any fixed $K_n$-nearest neighbor path $k^{\tau-1}=[k_0,\dots, k_{\tau-1}]$,
	\begin{equation}\label{eq:vanishing_distance_tau}
		\begin{split}
			d\Big( (X^{k^{\tau-1}}_{i_{\tau-1},t_{\tau-1}+1}&,u_{\tau}(X^{k^{\tau-1}}_{i_{\tau-1},t_{\tau-1}+1}))\\
			&,(X^{k^{\tau-1}}_{(K_n,n)},U^{k^{\tau-1}}_{(K_n,n)})((X^{k^{\tau-1}}_{i_{\tau-1},t_{\tau-1}+1},u_{\tau}(X^{k^{\tau-1}}_{i_{\tau-1},t_{\tau-1}+1})))\Big)\rightarrow0
		\end{split}
	\end{equation} 
	almost surely.
\end{lem}
\begin{proof}
	Let $\varepsilon>0$ and fix arbitrarily $k^{\tau-1}$, $\tau$ and $x_0$. $S(x,\varepsilon)$ denote the $\varepsilon$-ball centred at $x\in S$.
	For proving (\ref{eq:vanishing_distance_tau}), observe that
	\begin{align*}
		&\left\{d \left( (X^{k^{\tau-1}}_{i_{\tau-1},t_{\tau-1}+1},u_{\tau}(X^{k^{\tau-1}}_{i_{\tau-1},t_{\tau-1}+1})),(X^{k^{\tau-1}}_{(K_n,n)},U^{k^{\tau-1}}_{(K_n,n)})\right)>\varepsilon\right\}\\
		&\quad=\left\{\frac{1}{nT}\sum_{i=1}^{n}\sum_{t=0}^{T-1}I_{(X_{i,t},U_{i,t})\in S\left((X^{k^{\tau-1}}_{i_{\tau-1},t_{\tau-1}+1},u_{\tau}(X^{k^{\tau-1}}_{i_{\tau-1},t_{\tau-1}+1})),\varepsilon\right)}<\frac{k_n}{nT} \right\}\eqqcolon M_{\tau}
	\end{align*} 
	where $k^{\tau-1}$ and $\tau$ are fixed and 
	\begin{align*}
		M_{\tau}\subseteq\bigcup_{t=0}^{T-1}\left\{\frac{1}{nT}\sum_{i=1}^{n}I_{(X_{i,t},U_{i,t})\in S\left((X^{k^{\tau-1}}_{i_{\tau-1},t_{\tau-1}+1},u_{\tau}(X^{k^{\tau-1}}_{i_{\tau-1},t_{\tau-1}+1})),\varepsilon\right)}<\frac{k_n}{nT} \right\}.
	\end{align*}
	For any $\tilde{x}\in\bigcup_{t=0}^{T-1}supp(\pi_t|_x)$, 
	\begin{align*}
		&\pi(M_0|X^{k^{\tau-1}}_{i_{\tau-1},t_{\tau-1}+1}=\tilde{x} )\\
		&\leq\sum_{t=0}^{T-1}\pi_t\Bigg(\left\{\frac{1}{nT}\sum_{i=1}^{n}I_{(X_{i,t},U_{i,t})\in S\left((X^{k^{\tau-1}}_{i_{\tau-1},t_{\tau-1}+1},u_{\tau}(X^{k^{\tau-1}}_{i_{\tau-1},t_{\tau-1}+1})),\varepsilon\right)}<\frac{k_n}{nT} \right\}\\
		&\qquad\Bigg|X^{k^{\tau-1}}_{i_{\tau-1},t_{\tau-1}+1}=\tilde{x}\Bigg).
	\end{align*}
	where $\pi(\cdot|X^{k^{\tau-1}}_{i_{\tau-1},t_{\tau-1}+1}=\tilde{x})$ and  $\pi_t(\cdot|X^{k^{\tau-1}}_{i_{\tau-1},t_{\tau-1}+1}=\tilde{x})$ for all $t$ denotes the conditional probability measure. Note that by the strong law of large number
	\begin{align*}
		\sum_{t=0}^{T-1}&\frac{1}{n}\sum_{i=1}^{n}I_{(X_{i,t},U_{i,t})\in S\left((\tilde{x}, u_{\tau}(\tilde{x})),\varepsilon\right)}\\ 
		&\rightarrow \sum_{t=0}^{T-1}\pi_t\left(S\left((X^{k^{\tau-1}}_{i_{\tau-1},t_{\tau-1}+1}, u_{\tau}(X^{k^{\tau-1}}_{i_{\tau-1},t_{\tau-1}+1})),\varepsilon\right)\Bigg|X^{k^{\tau-1}}_{i_{\tau-1},t_{\tau-1}+1}=\tilde{x}\right) 
	\end{align*}
	almost surely and that by Assumption 1,
	\begin{align*}
		\sum_{t=0}^{T-1}\pi_t\left(S\left((X^{k^{\tau-1}}_{i_{\tau-1},t_{\tau-1}+1}, u_{\tau}(X^{k^{\tau-1}}_{i_{\tau-1},t_{\tau-1}+1})),\varepsilon\right)\Bigg|X^{k^{\tau-1}}_{i_{\tau-1},t_{\tau-1}+1}=\tilde{x}\right) > 0.
	\end{align*}
	Since by assumption, $K_n/n\rightarrow0$, it holds that 
	\begin{equation*}
		d \left( (\tilde{x},u_{\tau}(\tilde{x})),((X^{k^{\tau-1}}_{(K_n,n)},U^{k^{\tau-1}}_{(K_n,n)}))(\tilde{x},u_{\tau}(\tilde{x}))\right)\rightarrow0
	\end{equation*}
	almost surely for all  $\tilde{x}\in\bigcup_{t=0}^{T-1}supp(\pi_t|_x)$. Let $\mu_{X^{k^{\tau-1}}_{i_{\tau-1},t_{\tau-1}+1}}$ be the measure describing $X^{k^{\tau-1}}_{i_{\tau-1},t_{\tau-1}+1}$ then it holds that $supp(\mu_{X^{k^{\tau-1}}_{i_{\tau-1},t_{\tau-1}+1}})\subseteq\bigcup_{t=0}^{T-1}supp(\pi_t|_x)$. Thus, (\ref{eq:vanishing_distance_tau}) holds. (\ref{eq:vanishing_distance_0}) is shown analogously. 
\end{proof}

\begin{proof}[Proof of Theorem \ref{thm:NN_consistent_mult_main}]
For proving this theorem we check the assumptions of Theorem \ref{thm:stone_mult}. Assumption \ref{ass:weights_mult}.2 and Assumption  \ref{ass:weights_mult}.4 hold automatically by the definition of the weights. Proposition \ref{prop:vanishing_weights}.1 implies Assumption  \ref{ass:weights_mult}.1. Proposition \ref{prop:vanishing_weights}.2 and Proposition \ref{prop:vanishing_weights}.3 imply Assumption \ref{ass:weights_mult}.5. Let $\varepsilon>0$, then for all $\tau=0,\dots,T-1$ it holds that
\begin{align*}
	\mathbb{E}\Bigg[&\sum_{i^{\tau}, t^{\tau}}^{n, T} V_{n,i^{\tau}, t^{\tau}}
	I_{d\left((X_{i_{0},t_{0}},U_{i_{0},t_{0}}),(\tilde{X}_{1}, u_0(\tilde{X}_{1}))\right)>\varepsilon}\Bigg]\\
	=&\int_{\mathcal{X}}\mathbb{E}\Bigg[\sum_{i^{\tau}, t^{\tau}}^{n, T} V_{n,i^{\tau}, t^{\tau}}
	I_{d\left((X_{i_{0},t_{0}},U_{i_{0},t_{0}}),(x, u_0(x))\right)>\varepsilon}\Bigg]\pi_0(dx)\\
	=&\int_{\mathcal{X}}\mathbb{E}\Bigg[\frac{1}{K^{\tau}_n}\sum_{i^{\tau}, t^{\tau}}^{n, T} I_{i^{\tau}, t^{\tau} \text{ is a } K_n\text{-NN path}}
	I_{d\left((X_{i_{0},t_{0}},U_{i_{0},t_{0}}),(x, u_0(x))\right)>\varepsilon}\Bigg]\pi_0(dx)\\
	=& \int_{\mathcal{X}}\mathbb{E}\Bigg[\frac{1}{K^{\tau}_n}\sum_{k^{\tau}}^{K_n} 
	I_{d\left((X^{k^{\tau}}_{0},U^{k^{\tau}}_{0})(x, u_0(x)),(x, u_0(x))\right)>\varepsilon}\Bigg]\pi_0(dx)\rightarrow0,
\end{align*}
if 
\begin{equation*}
	\int_{\mathcal{X}}\mathbb{P}\left(d\left( (x, u_0(x)),(X_{(K_n,n)},U_{(K_n,n)})(x, u_0(x))\right) \right)\pi_0(dx).
\end{equation*}
This holds since Lemma \ref{lem:vanish_dist} implies 
\begin{equation*}
	\mathbb{P}\left(d\left( (x, u_0(x)),(X_{(K_n,n)},U_{(K_n,n)})(x, u_0(x))\right) \right)\rightarrow0
\end{equation*}
for all $x\in supp(\pi_0)$. Note for some $X_{t}^{k^{\tau}}$ corresponding to some $X_{j,s}$ in the original indexing, we denote $X_{j,s+1}$ by  $X_{t,+}^{k^{\tau}}$ and by $U_{t,+}^{k^{\tau}}$ we denote $u_{t+1}(X_{j,s+1})$. For any $l=0,\dots,T-1$, we have that
\begin{align*}
	\mathbb{E}\Bigg[&\sum_{i^{\tau}, t^{\tau}}^{n, T} V_{n,i^{\tau}, t^{\tau}}(\tilde{X}_{1}) I_{\tau-l}(\varepsilon)\Bigg]=\int_{\mathcal{X}}\mathbb{E}\Bigg[\sum_{i^{\tau}, t^{\tau}}^{n, T} V_{n,i^{\tau}, t^{\tau}}(x)I_{\tau-l}(\varepsilon)\Bigg]\pi_0(dx)\\
	=&\int_{\mathcal{X}}\mathbb{E}\Bigg[\frac{1}{K^{\tau}_n}\sum_{i^{\tau}, t^{\tau}}^{n, T} I_{i^{\tau}, t^{\tau} \text{ is a } K_n\text{-NN path}}(x)
	I_{\tau-l}(\varepsilon)\Bigg]\pi_0(dx)\\
	=& \int_{\mathcal{X}}\mathbb{E}\Bigg[\frac{1}{K^{\tau}_n} \sum_{k^{\tau}}^{K_n} 
	I_{d\left((X^{k^{\tau}}_{\tau-l},U^{k^{\tau}}_{\tau-l})(x),(X^{k^{\tau}}_{\tau-l-1,+},U^{k^{\tau}}_{\tau-l-1,+})(x)\right)>\varepsilon}\Bigg]\pi_0(dx)\rightarrow0
\end{align*}
if for any sequence $k^\tau$
\begin{equation*}\int_{\mathcal{X}}
	\mathbb{P}\left(d\left((X^{(k^{\tau-l-1}, K_n,k^{\tau-l+1:\tau})}_{\tau-l},U^{k^{\tau}}_{\tau-l})(x),(X^{k^{\tau}}_{\tau-l-1,+},U^{k^{\tau}}_{\tau-l-1,+})(x)\right)
	\right)\pi_0(dx)\rightarrow0
\end{equation*}
where $(k^{\tau-l-1}, K_n,k^{\tau-l+1:\tau}) =(k^\tau_0,\dots,k^\tau_{\tau-l-1},K_n,k^\tau_{\tau-l+1}, \dots,k^\tau_{\tau})$. This since Lemma \ref{lem:vanish_dist}, implies
\begin{equation*}
	\mathbb{P}\left(d\left((X^{(k^{\tau-l-1}, K_n,k^{\tau-l+1:\tau})}_{\tau-l},U^{k^{\tau}}_{\tau-l})(x),(X^{k^{\tau}}_{\tau-l-1,+},U^{k^{\tau}}_{\tau-l-1,+})(x)\right)
	\right)\rightarrow0
\end{equation*}
for all $x\in supp(\pi_0)$. Thus, Assumption \ref{ass:weights_mult}.3 is satisfied. 
\end{proof}

\subsection{Consistency of the KNNR estimator}\label{subsec:resampling_main}
\begin{proof}[Proof of Theorem \ref{thm:resampling_main}]
	Consider
	\begin{align*}
		\mathbb{E}&\left[\left(\msub_{l_n,n}(u)-m(u)\right)^2\right]\\
		&\leq 2\mathbb{E}\left[\left(\msub_{l_n,n}(u)-\mnn_{l_n,n}(u)\right)^2\right] +2\mathbb{E}\left[\left(\mnn_{l_n,n}(u)-m(u)\right)^2\right]
	\end{align*}	
	where second term vanishes as $n\rightarrow\infty$ by Theorem 3.
	For the first term,
	\begin{align*}
		\mathbb{E}\left[\left(\sum_{\tau=0}^{T}\msub_{l_n,n}(u,\tau)-\mnn_{l_n,n}(u,\tau)\right)^2\right]\\
		\leq (T+1)\sum_{\tau=0}^{T}\mathbb{E}\left[\left(\msub_{l_n,n}(u,\tau)-\mnn_{l_n,n}(u,\tau)\right)^2\right].
	\end{align*}
	
	It will be sufficient to show that
	\begin{equation*}
		\lim_{n\rightarrow\infty}\mathbb{E}\left[\left(\msub_{l_n,n}(u,\tau)-\mnn_{l_n,n}(u,\tau)\right)^2\right]=0
	\end{equation*}
	for a fixed $\tau$ since the arguments are analogous for the other terms.
	Consider
	\begin{align*}
		\mathbb{E}&\left[\left(
		\frac{1}{l_n}\sum_{j=1}^{l_n} \sum_{k^{\tau}}^{K_n} r_{\tau}^{k^{\tau}}(\tilde{X}_j,u)(I_{k^{\tau}=\mathcal{K}^{\tau}_j}-\frac{1}{K^{\tau}_n})\right)^2\right]\\
		&=\frac{1}{l^2_n}\mathbb{E}\left[
		\sum_{j_1=1}^{l_n} \sum_{j_2=1}^{l_n} \sum_{k_1^{\tau}}^{K_n} \sum_{k_2^{\tau}}^{K_n}r_{\tau}^{k_1^{\tau}}(\tilde{X}_{j_1},u)r_{\tau}^{k_2^{\tau}}(\tilde{X}_{j_2},u)(I_{k_1^{\tau}=\mathcal{K}^{\tau}_{j_1}}-\frac{1}{K^{\tau}_n})(I_{k_2^{\tau}=\mathcal{K}^{\tau}_{j_2}}-\frac{1}{K^{\tau}_n})\right]\\
		&=I_n
	\end{align*}
	Since $\mathbb{E}[I_{k^{\tau}=\mathcal{K}^{\tau}_{j}}]=\frac{1}{K^{\tau}_n}$ for all $j$ and due to the independence assumptions on $\mathcal{K}^{\tau}_{j}$, we have for $j_1\neq j_2$ and any $k_1^{\tau}$ and $k_2^{\tau}$:
	
	\begin{align*}
		\mathbb{E}&\left[ r_{\tau}^{k_1^{\tau}}(\tilde{X}_{j_1},u)r_{\tau}^{k_2^{\tau}}(\tilde{X}_{j_2},u)(I_{k_1^{\tau}=\mathcal{K}^{\tau}_{j_1}}-\frac{1}{K^{\tau}_n})(I_{k_2^{\tau}=\mathcal{K}^{\tau}_{j_2}}-\frac{1}{K^{\tau}_n})\right]\\
		&=\mathbb{E}\left[ r_{\tau}^{k_1^{\tau}}(\tilde{X}_{j_1},u)r_{\tau}^{k_2^{\tau}}(\tilde{X}_{j_2},u)\right]\mathbb{E}\left[ I_{k_1^{\tau}=\mathcal{K}^{\tau}_{j_1}}-\frac{1}{K^{\tau}_n}\right]\mathbb{E}\left[I_{k_2^{\tau}=\mathcal{K}^{\tau}_{j_2}}-\frac{1}{K^{\tau}_n}\right] \\
		&=0
	\end{align*}
	For $j_1=j_2$ and $k_1^{\tau}\neq k_2^{\tau}$,
	\begin{align*}
		\mathbb{E}&\left[ r_{\tau}^{k_1^{\tau}}(\tilde{X}_{j_1},u)r_{\tau}^{k_2^{\tau}}(\tilde{X}_{j_2},u)(I_{k_1^{\tau}=\mathcal{K}^{\tau}_{j_1}}-\frac{1}{K^{\tau}_n})(I_{k_2^{\tau}=\mathcal{K}^{\tau}_{j_2}}-\frac{1}{K^{\tau}_n})\right]\\
		&=\mathbb{E}\left[ r_{\tau}^{k_1^{\tau}}(\tilde{X}_{j_1},u)r_{\tau}^{k_2^{\tau}}(\tilde{X}_{j_2},u)\right]\left(-\frac{1}{K^{2\tau}_n}\right)
	\end{align*}
	since $\mathbb{E}[I_{k_1^{\tau}=\mathcal{K}^{\tau}_{j_1}}I_{k_2^{\tau}=\mathcal{K}^{\tau}_{j_2}}]=0$.
	For $j_1=j_2$ and $k_1^{\tau}= k_2^{\tau}$,
	\begin{align*}
		\mathbb{E}&\left[ r_{\tau}^{k_1^{\tau}}(\tilde{X}_{j_1},u)r_{\tau}^{k_2^{\tau}}(\tilde{X}_{j_2},u)(I_{k_1^{\tau}=\mathcal{K}^{\tau}_{j_1}}-\frac{1}{K^{\tau}_n})(I_{k_2^{\tau}=\mathcal{K}^{\tau}_{j_2}}-\frac{1}{K^{\tau}_n})\right]\\
		&=\mathbb{E}\left[(r_{\tau}^{k_1^{\tau}}(\tilde{X}_{j_1},u))^2\right] \left(\frac{1}{K^{\tau}_n}\right)
	\end{align*}
	where we used that $I_{k^{\tau}=\mathcal{K}^{\tau}_{j}}$ is Bernoulli distributed. Since for all $k_1^{\tau}, k_2^{\tau}$ and for all $j_1,j_2$, it holds that $\mathbb{E}\left[(r_{\tau}^{k_1^{\tau}}(\tilde{X}_{j_1},u))^2\right]\leq4M^2$ and $\mathbb{E}\left[|r_{\tau}^{k_1^{\tau}}(\tilde{X}_{j_1},u)r_{\tau}^{k_2^{\tau}}(\tilde{X}_{j_2},u)|\right]\leq4M^2$. Thus,
	
	\begin{align*}
		&I_n\leq \frac{1}{l^2_n}\sum_{j=1}^{l_n}\Bigg(\sum_{k^{\tau}}^{K_n} \mathbb{E}\left[(r_{\tau}^{k_1^{\tau}}(\tilde{X}_{j},u))^2\right] \left(\frac{1}{K^{\tau}_n}\right)\\
		&\qquad+\sum_{k_1^{\tau}}^{K_n} \sum_{\substack{k_2^{\tau}\\k_1^{\tau}\neq k_2^{\tau}}}^{K_n} \mathbb{E}\left[|r_{\tau}^{k_1^{\tau}}(\tilde{X}_{j},u)r_{\tau}^{k_2^{\tau}}(\tilde{X}_{j},u)|\right]\left(\frac{1}{K^{2\tau}_n} \right)\Bigg)\\
		&\leq \frac{1}{l^2_n} \sum_{j=1}^{l_n}8M^2 = \frac{8M^2 }{l_n}\rightarrow0
	\end{align*}
	as $n\rightarrow\infty$.
\end{proof}

%% file: 20230507conclusion.tex
\section{Conclusion}

We believe that this work offers multiple possibilities for future research. First, bridging between reinforcement learning and stochastic control, we expect that the $K$-NN resampling algorithm can be applied to a multitude of practical stochastic control problems in diverse fields such as finance, operations research or engineering. Another interesting avenue of research is how to choose the state representation for our algorithm. With the right state representation the curse of dimensionality can be overcome in the nearest neighbor search \cite{pari2022surprising} or it could be applied to non-Markovian control problems. Related to this, the metric that is used for nearest neighbor search can be viewed as an engineering choice and potentially improved performance can be seen with a learned metric as in \cite{dadashi2021offline}. Our generalization of the local averaging theorem can also be employed to guide the design of other OPE methods based on local averaging. This could include kernel-based or other partitioning-based approaches.

%% file: 20231121theory.tex
\section{Theoretical Considerations}
 
\subsection{Environment}
In this section, we discuss the relationship of our setting to general MDPs and the additional restrictions we impose on the dynamical system.

\paragraph{Comparison: MDP vs Controlled Iterated Function System (IFS)}
The standard MDP formulation for an episodic setting of fixed length can be related to our random dynamical systems formulation in the following sense: 
Consider an MDP $(\mathcal{X}, \mathcal{U}, \mathbb{P}, R)$ with a fixed episode length $T$ where $\mathcal{X}$ is the state space and $\mathcal{U}$ the action space. $P\left(\cdot \mid x_t, a_t\right)$ denotes the transition kernel, $R\left(x_t, a_t,x_{t+1}\right)$ the reward distribution, $d_0$ the distribution of the initial state and some Markovian policy $\pi\left(\cdot \mid x_t\right)$. The value of the policy is then given by    $V\left(\pi\right)=\mathbb{E}_{x \sim d_0}\left[\sum_{t=0}^{T-1} r_t \mid x_0=x\right]$ where $a_t \sim \pi_e\left(\cdot \mid x_t\right), x_{t+1} \sim P\left(\cdot \mid x_t, a_t\right), r_t \sim R\left(x_t, a_t,x_{t+1}\right)$. 

Further, recall that an MDP with a Markovian policy reduces to a time-homogenous Markov process on $S=\mathcal{X}\times\mathcal{U}$ where $S$ is equipped with a $\sigma$-algebra $\mathcal{S}$ and the transition kernels are denoted by $(x_{t+1},a_{t+1}) \sim \tilde{P}\left(\cdot \mid x_t, a_t\right)$. It is a well known fact that every time-homogenous Markov process can be generated via a random dynamical system with iid random functions that are measurable with respect to $(S,\mathcal{S})$ i.e. \cite{arnold1995random,bhattacharya2009stochastic}. The representation of a Markov process with random functions however is not unique. Furthermore, how restrictions on the distribution of the random functions (i.e. only continuous functions or sets with equicontinuous functions \footnote{Consider a set of functions $\Gamma$ between two metric spaces $(X,d_1)$ and $(Y,d_2)$.  $\Gamma$ consists of uniformly equicontinuous functions if for every $\varepsilon>0$, there exists a $\delta>0$ such that $d_2(f(x_1),f(x_2))<\varepsilon$ for all $f\in\Gamma$ and all $x_1,x_2\in X$ where $d_1(x_1,x_2)<\delta$. }) impact the scope of representable Markov processes in general is to our knowledge an open research question and is beyond the scope of our paper. For the one dimensional case (\cite{bhattacharya2009stochastic} p.228), one can explicitly construct an iterated function system from the quantile function of the transition kernel. With the same method one can construct an iterated function system where the functions are continuous or from a set of equicontinuous functions from the quantile function of the transition kernel where the quantile function satisfies the same property. While it is unclear how the representative power of an iterative function system varies between distributions on $C(S,S)$ and equicontinuous subsets $\Gamma\subset C(S,S)$, for the case that $S$ is compact we know that by Arzela-Ascoli's theorem (e.g. see \cite{alt2016linear}) every compact set contained in  $C(S,S)$ is uniformly equicontinuous. Finally, note that the setting we propose is more general compared to the settings considered in \cite{fonteneau2010model,fonteneau2013batch, gottesman2019combining}.

\paragraph{Assumptions on the dynamical system}
In general, we make three additional assumptions on the dynamical system:
\begin{ass}
	\begin{enumerate}
		\item $\Gamma\subset C(S,S)$ is a set of uniformly equicontinuous functions
		\item  $\epsilon_t$ is supported on $[-M, M]$ for $M\in\mathbb{R}$ and it holds that $r:S\times\mathcal{X}\rightarrow [-M,M]$
		\item $(\epsilon_t)_{t=0}^T$ and $\left(\alpha_t\right)_{t=1}^{T}$ are assumed to be independent and $\mathbb{E}[\epsilon_t]=0$
	\end{enumerate}
\end{ass}
The first assumption becomes necessary due to taking limits over compositions of transition functions. The impact of this assumption on the generality of the results is discussed in the previous paragraph and we note that sets of uniformly Lipschitz continuous functions (as in \cite{fonteneau2010model,fonteneau2013batch, gottesman2019combining}) are in particular uniformly equicontinuous. The second assumption simplifies the proof considerably and the restriction on regression targets on a compact set is a common simplification in non-parametric statistics. It is possible that this assumption is not necessary as the nearest neighbor regression satisfies a strong concentration property (cf. \cite{gyorfi2002distribution} Lemma 6.3). Possibly this could be extended to our setting. The third assumption is necessary to disentangle the noise of the transition from averaging over the rewards. Note however with the generalization mentioned in Remark 1 of the main paper one can introduce rewards depending on the transitions which are more general than the commonly assumed reward structures in a dynamical system formulation (i.e. \cite{bertsekas2012dynamic}). Finally, the condition $\mathbb{E}[\epsilon_t]=0$ can always be satisfied by shifting the reward function by the expectation of the noise.

\subsection{Proof of Propositions \ref{prop:stone_mult}}

\begin{prop}\label{prop:stone_mult}
	Given the setting described in Section 2 of the main paper and under Assumption \ref{ass:weights_mult}, it holds that $D_n$, $E_n$, $F_n$, and $G_n$ from the proof of Theorem \ref{thm:stone_mult} vanish as $n\rightarrow\infty$.
\end{prop}

\begin{proof}	
For 
\begin{align*}
	E_n =& \mathbb{E}\Biggl[\Biggl(\sum_{i^{\tau}, t^{\tau}}^{n, T}\sum_{\substack{(j^{\tau-1},i_{\tau}),\\ (s^{\tau-1},t_{\tau})}}^{n, T} W_{n,i^{\tau}, t^{\tau}} W_{n,(j^{\tau-1},i_{\tau}), (s^{\tau-1},t_{\tau})}(R_{i_{\tau},t_{\tau}}-r(X_{i_{\tau},t_{\tau}},U_{i_{\tau},t_{\tau}}))^2\Biggr)\Biggr]	\\
	+&\mathbb{E}\Biggl[\Biggl(\sum_{i^{\tau}, t^{\tau}}^{n, T}\sum_{\substack{j^{\tau}, s^{\tau}\\ j_{\tau}\neq i_{\tau}\vee s_{\tau}\neq t_{\tau} }}^{N, T} W_{n,i^{\tau}, t^{\tau}} W_{n,j^{\tau}, s^{\tau}}(R_{i_{\tau},t_{\tau}}-r(X_{i_{\tau},t_{\tau}},U_{i_{\tau},t_{\tau}}))\\
	&\qquad\qquad\qquad\cdot(R_{j_{\tau},s_{\tau}}-r(X_{j_{\tau},s_{\tau}},U_{j_{\tau},s_{\tau}}))\Biggr)\Biggr]\\
	\coloneqq& E'_n + E_n''	
\end{align*}
note first that by Assumption \ref{ass:weights_mult}.1, 
\begin{align*}
|E'_n|\leq 4M^2 \mathbb{E}\Bigg[\sum_{i^{\tau}, t^{\tau}}^{n, T}\sum_{\substack{(j^{\tau-1},i_{\tau}),\\ (s^{\tau-1},t_{\tau})}}^{n, T} |W_{n,i^{\tau}, t^{\tau}} W_{n,(j^{\tau-1},i_{\tau}), (s^{\tau-1},t_{\tau})}|\Bigg]\rightarrow 0
\end{align*}
and 
\begin{align*}
	E_n'' = 0
\end{align*}
since for $i_\tau\neq j_\tau$ or $s_\tau \neq t_\tau$
\begin{align*}
	\mathbb{E}&\Bigl[W_{n,i^{\tau}, t^{\tau}} W_{n,j^{\tau}, s^{\tau}}(R_{i_{\tau},t_{\tau}}-r(X_{i_{\tau},t_{\tau}},U_{i_{\tau},t_{\tau}})) (R_{j_{\tau},s_{\tau}}-r(X_{j_{\tau},s_{\tau}},U_{j_{\tau},s_{\tau}}))\Bigr]\\
	\leq& \mathbb{E}\Bigl[\mathbb{E}\Bigl[W_{n,i^{\tau}, t^{\tau}} W_{n,j^{\tau}, s^{\tau}}\epsilon_{i_{\tau},t_{\tau}}\epsilon_{j_{\tau},s_{\tau}}\Big| D\Bigr]\Bigr]\\
	\leq&\mathbb{E}\Bigl[W_{n,i^{\tau}, t^{\tau}} W_{n,j^{\tau}, s^{\tau}}\mathbb{E}[\epsilon_{i_{\tau},t_{\tau}} ]\mathbb{E}[\epsilon_{j_{\tau},t_{\tau}}]\Bigr]\\
	=&0
\end{align*}
by the definition of $\epsilon_{i_{\tau},t_{\tau}}$ and $\epsilon_{j_{\tau},s_{\tau}}$.

Next, $F_n$ can be represented by a telescoping sum, such that
\begin{align*}
F_n &\leq (\tau+2)\mathbb{E}\Bigg[\Bigg(\sum_{i^{\tau}, t^{\tau}}^{n, T} W_{n,i^{\tau}, t^{\tau}}\big(r(X_{i_{\tau},t_{\tau}},U_{i_{\tau},t_{\tau}})\\
&\qquad-r(\alpha_{i_{\tau-1},t_{\tau-1}}|_u(X_{i_{\tau-1},t_{\tau-1}},U_{i_{\tau-1},t_{\tau-1}}))\big)\Bigg)^2 \Bigg]\\
&+ (\tau+2)\mathbb{E}\Bigg[\sum_{k=1}^{\tau-1}\\
&\quad\Bigg(\sum_{i^{\tau}, t^{\tau}}^{n, T} W_{n,i^{\tau}, t^{\tau}}\big(r(\alpha_{i_{\tau-1},t_{\tau-1}}|_u\circ\cdots\circ\alpha_{i_{\tau-k},t_{\tau-k}}|_u(X_{i_{\tau-k},t_{\tau-k}},U_{i_{\tau-k},t_{\tau-k}}))\\
& \qquad- r(\alpha_{i_{\tau-1},t_{\tau-1}}|_u\circ\cdots\circ\alpha_{i_{\tau-k-1},t_{\tau-k-1}}|_u(X_{i_{\tau-k-1},t_{\tau-k-1}},U_{i_{\tau-k-1},t_{\tau-k-1}}))\big)\Bigg)^2\Bigg]\\
&+ (\tau+2)\mathbb{E}\Bigg[\Bigg(\sum_{i^{\tau}, t^{\tau}}^{n, T} W_{n,i^{\tau}, t^{\tau}}\big(r(\alpha_{i_{\tau-1},t_{\tau-1}}|_u\circ\cdots\circ\alpha_{i_{0},t_{0}}|_u(X_{i_{0},t_{0}},U_{i_{0},t_{0}}))\\
& \quad- r(\alpha_{i_{\tau-1},t_{\tau-1}}|_u\circ\cdots\circ\alpha_{i_{0},t_{0}}|_u(\tilde{X}_{1},u_0(\tilde{X}_{1})))\big)\Bigg)^2\Bigg]\\
&\coloneqq(\tau+2)\left(F^{\tau}_n +\sum_{k=1}^{\tau}F^{\tau-k}_n+ F^{-1}_n\right).
\end{align*} 

By Cauchy--Schwarz inequality, and Assumption \ref{ass:weights_mult}.2, 

\begin{align*}
F^{\tau}_n&\leq\mathbb{E}\Bigg[\Bigg(\sum_{i^{\tau}, t^{\tau}}^{n, T} \sqrt{ |W_{n,i^{\tau}, t^{\tau}}|}\sqrt{| W_{n,i^{\tau}, t^{\tau}}|}\\
&\qquad\qquad\big|r(X_{i_{\tau},t_{\tau}},U_{i_{\tau},t_{\tau}})-r(\alpha_{i_{\tau-1},t_{\tau-1}}|_u(X_{i_{\tau-1},t_{\tau-1}},U_{i_{\tau-1},t_{\tau-1}}))\big|\Bigg)^2 \Bigg]\\
&\leq \mathbb{E}\Bigg[\Bigg(\sum_{i^{\tau}, t^{\tau}}^{n, T}  |W_{n,i^{\tau}, t^{\tau}}|\Bigg)\\
&\qquad\cdot\Bigg( \sum_{i^{\tau}, t^{\tau}}^{n, T}| W_{n,i^{\tau}, t^{\tau}}|
\Bigg(
r(X_{i_{\tau},t_{\tau}},U_{i_{\tau},t_{\tau}})-r(X_{i_{\tau-1},t_{\tau-1}+1},u_{\tau}(X_{i_{\tau-1},t_{\tau-1}+1}))
\Bigg)^2
\Bigg)\Bigg]\\
&\leq L\mathbb{E}\Bigg[\sum_{i^{\tau}, t^{\tau}}^{n, T}| W_{n,i^{\tau}, t^{\tau}}|
\Bigg(
r(X_{i_{\tau},t_{\tau}},U_{i_{\tau},t_{\tau}})-r(X_{i_{\tau-1},t_{\tau-1}+1},u_{\tau}(X_{i_{\tau-1},t_{\tau-1}+1}))
\Bigg)^2\Bigg]\\
&=L \tilde{F}^{\tau}_n
\end{align*}
and,  for an arbitrary $\delta>0$,
\begin{align*}
\tilde{F}^{\tau}_n\leq&  \mathbb{E}\Bigg[\sum_{i^{\tau}, t^{\tau}}^{n, T}| W_{n,i^{\tau}, t^{\tau}}|
\left(
r(X_{i_{\tau},t_{\tau}},U_{i_{\tau},t_{\tau}})-r(X_{i_{\tau-1},t_{\tau-1}+1},u_{\tau}(X_{i_{\tau-1},t_{\tau-1}+1}))
\right)^2\\
&\qquad\cdot I_{d\left((X_{i_{\tau},t_{\tau}},U_{i_{\tau},t_{\tau}}),(X_{i_{\tau-1},t_{\tau-1}+1}, u_{\tau}(X_{i_{\tau-1},t_{\tau-1}+1}))\right) \leq\delta}
\Bigg]\\
+& \mathbb{E}\Bigg[\sum_{i^{\tau}, t^{\tau}}^{n, T}| W_{n,i^{\tau}, t^{\tau}}|
\left(
r(X_{i_{\tau},t_{\tau}},U_{i_{\tau},t_{\tau}})-r(X_{i_{\tau-1},t_{\tau-1}+1},u_{\tau}(X_{i_{\tau-1},t_{\tau-1}+1}))
\right)^2\\
&\qquad\cdot I_{d\left((X_{i_{\tau},t_{\tau}},U_{i_{\tau},t_{\tau}}),(X_{i_{\tau-1},t_{\tau-1}+1}, u_{\tau}(X_{i_{\tau-1},t_{\tau-1}+1}))\right) >\delta}
\Bigg]\\
\leq&L \left(\sup_{x,y\in S:d(x,y)\leq\delta}\left|r(x)-r(y)\right|\right)^2\\
&+ 4M^2\mathbb{E}\left[\sum_{i^{\tau}, t^{\tau}}^{n, T}| W_{n,i^{\tau}, t^{\tau}}|
I_{d\left((X_{i_{\tau},t_{\tau}},U_{i_{\tau},t_{\tau}}),(X_{i_{\tau-1},t_{\tau-1}+1}, u_{\tau}(X_{i_{\tau-1},t_{\tau-1}+1}))\right) >\delta}\right]
\end{align*}

Hence, by  Assumption \ref{ass:weights_mult}.3,
\begin{equation*}
	\limsup_{n\rightarrow\infty} F^{\tau}_n\leq L^2\left(\sup_{x,y\in S :d(x,y)\leq\delta}\left|r(x)-r(y)\right|\right)^2
\end{equation*}
and since $r$ uniformly continuous  
\begin{equation*}
F^{\tau}_n\rightarrow0
\end{equation*}
as $\delta\rightarrow0$. 
The argument for $F^{\tau-k}_n$ with $k=1,...,\tau+1$ is analogous additionally using the uniform equicontinuity of the random functions.

Consider
\begin{align*}
	G_n \leq & 3\mathbb{E}\Bigg[\Bigg(\sum_{i^{\tau}, t^{\tau}}^{n, T} \Big(W_{n,i^{\tau}, t^{\tau}}\\
	&\quad -W_{n,i^{\tau}, t^{\tau}}(D_{i^{\tau},t^{\tau}})\Big)r(\alpha_{i_{\tau},t_{\tau}}|_u\circ\cdots\circ\alpha_{i_{0},t_{0}}|_u(\tilde{X}_{1},u_0(\tilde{X}_{1})))\Bigg)^2\Bigg]\\
	&+3\mathbb{E}\Biggl[\Biggl(\sum_{i^{\tau}, t^{\tau}}^{n, T} W_{n,i^{\tau}, t^{\tau}}(D_{i^{\tau},t^{\tau}})\Delta_{i^{\tau},t^{\tau}}r(\tilde{X}_{1})\Biggr)^2 \Biggr]\\
	& +3\mathbb{E}\Bigg[\Bigg(\Bigg(\sum_{i^{\tau}, t^{\tau}}^{n, T} W_{n,i^{\tau}, t^{\tau}}(D_{i^{\tau},t^{\tau}})-1\Bigg)\mathbb{E}[R_{\tau}|_u| \tilde{X}_{1}]\Bigg)^2 \Bigg]\\
	\eqqcolon& 3H_n + 3I_n + 3J_n   
\end{align*}
where
\begin{align*}
	\Delta_{i^{\tau},t^{\tau}}r(\tilde{X}_{1})\coloneqq \biggl(r(\alpha_{i_{\tau-1},t_{\tau-1}}|_u\circ\cdots\circ\alpha_{i_{0},t_{0}}|_u(\tilde{X}_{1},u_0(\tilde{X}_{1})))-\mathbb{E}[R_{\tau}|_u| \tilde{X}_{1}]\biggr).
\end{align*}

Next, for
\begin{align*}
	H_n\leq& 4M^2 \mathbb{E}\Bigg[\Bigg(\sum_{i^{\tau}, t^{\tau}}^{n, T} |W_{n,i^{\tau}, t^{\tau}}-W_{n,i^{\tau}, t^{\tau}}(D_{i^{\tau},t^{\tau}})|\Bigg)^2\Bigg]
\end{align*}
and, thus, by Assumption \ref{ass:weights_mult}.5, $H_n\rightarrow0$.

Further,
\begin{align*}
	I_n =& \mathbb{E}\Biggl[\sum_{i^\tau,t^\tau}^{n,T}\sum_{\substack{j^\tau,s^\tau\\  |(i^\tau,j^\tau)|<2\tau+2}}^{n,T}	W_{n,i^{\tau}, t^{\tau}}(D_{i^{\tau},t^{\tau}})W_{n,j^{\tau}, s^{\tau}}(D_{j^{\tau},s^{\tau}})\\
	&\quad\Delta_{i^{\tau},t^{\tau}}r(\tilde{X}_{1})\Delta_{j^{\tau},s^{\tau}}r(\tilde{X}_{1})  \Biggr]\\
	&+ \mathbb{E}\Biggl[\sum_{i^\tau,t^\tau}^{n,T}\sum_{\substack{j^\tau,s^\tau\\  |(i^\tau,j^\tau)|=2\tau+2}}^{n,T}
	W_{n,i^{\tau}, t^{\tau}}(D_{i^{\tau},t^{\tau}})W_{n,j^{\tau}, s^{\tau}}(D_{j^{\tau},s^{\tau}})\\
	&\quad\Delta_{i^{\tau},t^{\tau}}r(\tilde{X}_{1})\Delta_{j^{\tau},s^{\tau}}r(\tilde{X}_{1})\Biggr]\\
	\coloneqq& I_n'+I_n''
\end{align*}
where by Assumption \ref{ass:weights_mult}.1
\begin{align*}
	I_n'\leq4M^2\mathbb{E}\Bigg[\sum_{i^\tau,t^\tau}^{n,T}\sum_{\substack{j^\tau,s^\tau\\  |(i^\tau,j^\tau)|<2\tau+2}}^{n,T}	|W_{n,i^{\tau}, t^{\tau}}(D_{i^{\tau},t^{\tau}})W_{n,j^{\tau}, s^{\tau}}(D_{j^{\tau},s^{\tau}})|\Bigg]\rightarrow 0.
\end{align*}

Note that
\begin{align*}
|I_n''|&\leq	
4M^2\mathbb{E}\Biggl[\sum_{i^\tau,t^\tau}^{n,T}\sum_{\substack{j^\tau,s^\tau\\  |(i^\tau,j^\tau)|=2\tau+2}}^{n,T}|W_{n,i^{\tau}, t^{\tau}}(D_{i^{\tau},t^{\tau}})|\\
&\quad\cdot|W_{n,j^{\tau}, s^{\tau}}(D_{j^{\tau},s^{\tau}})
-W_{n,j^{\tau}, s^{\tau}}(D_{(i^{\tau},j^{\tau}),(t^{\tau},s^{\tau})})|\Biggr]\\
&+\Bigg|\mathbb{E}\Biggl[\sum_{i^\tau,t^\tau}^{n,T}\sum_{\substack{j^\tau,s^\tau\\  |(i^\tau,j^\tau)|=2\tau+2}}^{n,T}W_{n,i^{\tau}, t^{\tau}}(D_{i^{\tau},t^{\tau}})
W_{n,j^{\tau}, s^{\tau}}(D_{(i^{\tau},j^{\tau}),(t^{\tau},s^{\tau})})\Biggr]\Bigg|\\
&\coloneqq \hat{I}_n'+|\hat{I}_n''|
\end{align*}
where by Assumption \ref{ass:weights_mult}.5  $\hat{I}_n'\rightarrow0$. For $\hat{I}_n''$, define $\Theta_{(j^{\tau},s^{\tau})} \coloneqq (\tilde{X}_{1}, (\alpha_{j_{k},s_{k}})_{k=0}^{\tau-1})$. Note that by definition $W_{n,i^{\tau}, t^{\tau}}(D_{i^{\tau},t^{\tau}})
W_{n,j^{\tau}, s^{\tau}}(D_{(i^{\tau},j^{\tau}),(t^{\tau},s^{\tau})})$ and $\Delta_{i^{\tau},t^{\tau}}r(\tilde{X}_{1})$ are independent conditioned on  $\Theta_{(j^{\tau},s^{\tau})}$ and that $\Delta_{i^{\tau},t^{\tau}}r(\tilde{X}_{1})$ is independent of $(\alpha_{j_{k},s_{k}})_{k=0}^{\tau-1}$. This observation with the measurability  of $\Delta_{j^{\tau},s^{\tau}}r(\tilde{X}_{1})$ with respect to the $\sigma$-algebra generated by $\Theta_{(j^{\tau},s^{\tau})}$ and the fact that $\mathbb{E}[\Delta_{i^{\tau},t^{\tau}}r(\tilde{X}_{1})\mid
\tilde{X}_{1}]=0$, yields for any summand where $|(i^\tau,j^\tau)|=\tau+2$ that 
\begin{align*}
\mathbb{E}&\Biggl[W_{n,i^{\tau}, t^{\tau}}(D_{i^{\tau},t^{\tau}})
W_{n,j^{\tau}, s^{\tau}}(D_{(i^{\tau},j^{\tau}),(t^{\tau},s^{\tau})})\Delta_{i^{\tau},t^{\tau}}r(\tilde{X}_{1})\Delta_{j^{\tau},s^{\tau}}r(\tilde{X}_{1})\Biggr]\\
=&\mathbb{E}\Biggl[\mathbb{E}\Bigl[W_{n,i^{\tau}, t^{\tau}}(D_{i^{\tau},t^{\tau}})
W_{n,j^{\tau}, s^{\tau}}(D_{(i^{\tau},j^{\tau}),(t^{\tau},s^{\tau})})\Delta_{i^{\tau},t^{\tau}}r(\tilde{X}_{1})\mid\Theta_{(j^{\tau},s^{\tau})}
\Bigr]\Delta_{j^{\tau},s^{\tau}}r(\tilde{X}_{1})
\Biggr]\\
=&\mathbb{E}\Biggl[\mathbb{E}\Bigl[W_{n,i^{\tau}, t^{\tau}}(D_{i^{\tau},t^{\tau}})
W_{n,j^{\tau}, s^{\tau}}(D_{(i^{\tau},j^{\tau}),(t^{\tau},s^{\tau})})\mid\Theta_{(j^{\tau},s^{\tau})}
\Bigr]\\
&\quad\cdot \mathbb{E}\Bigl[\Delta_{i^{\tau},t^{\tau}}r(\tilde{X}_{1})\mid \tilde{X}_{1}
\Bigr]\Delta_{j^{\tau},s^{\tau}}r(\tilde{X}_{1})
\Biggr]\\
=&0
\end{align*}

By Assumption \ref{ass:weights_mult}.2, \ref{ass:weights_mult}.4, \ref{ass:weights_mult}.5 and dominated convergence,
\begin{align*}
\lim_{n\rightarrow\infty}J_n  =\lim_{n\rightarrow\infty}\mathbb{E}\left[\left(\left(\sum_{i^{\tau}, t^{\tau}}^{n, T} W_{n,i^{\tau}, t^{\tau}}(D_{i^{\tau},t^{\tau}})-1\right)\mathbb{E}[R_{\tau}|_u| \tilde{X}_{1}]\right)^2 \right] =0.
\end{align*}
and for 
\begin{equation*}
	\lim_{n\rightarrow\infty}D_n = \lim_{n\rightarrow\infty}\frac{\sigma^2_{\mathbb{E}[R_{\tau}|_u| \tilde{X}_{j}]}}{l_n}=0
\end{equation*}
where $\sigma^2_{\mathbb{E}[R_{\tau}|_u| \tilde{X}_{j}]}$ is the variance of $\mathbb{E}[R_{\tau}|_u| \tilde{X}_{j}]$.
\end{proof}

\section{Proof of Proposition \ref{prop:vanishing_weights}}

This section will show that Assumptions \ref{ass:weights_mult}.1 and \ref{ass:weights_mult}.2 are satisfied by the $K_n$-NN path regression estimator.

This proposition in slightly varied forms is central to proving the consistency of the $K_n$-NN path regression estimator. To simplify readability we make several definitions: First, let $\rho\in\mathbb{N}$. Consider the multi-index for the data $j^{\rho},s^{\rho}$ and the data set $D_{j^{\rho},s^{\rho}}$ (as defined above). For any $l=0,\dots,\tau$ denote the product,
\begin{equation*}
	V_{n,i^{l:\tau}, t^{l:\tau}}(D_{j^{\rho},s^{\rho}})=\prod_{k=l}^{\tau} V_{n,[(i_{k-1}, t_k{k-1}),(i_{k}, t_{k})]}V_{n,[(i_{k-1}, t_{k-1}),(i_{k}, t_{k})]}(D_{j^{\rho},s^{\rho}}).
\end{equation*}
For $l=0,\dots\rho$, and the multi-index $i^{\tau},t^{\tau}$ denote 
\begin{equation*}
	\begin{split}
		&\tilde{V}_{n,[(i_{k-1}, t_{k-1}),(i_{k}, t_{k})]}(D_{j^{l+1:\rho},s^{l+1:\rho}}, D_{j^{l:\rho},s^{l:\rho}})\\
		&\quad=\frac{1}{K_n}I_{(X_{i_{k}, t_{k}},U_{i_{k}, t_{k}}) \text{ is a } K_n-\text{NN}\text{ in }D_{j^{l:\rho},s^{l:\rho}}}(X_{i_{k-1}, t_{k-1}+1},u_{t+1}(X_{i_{k-1}, t_{k-1}+1}))
	\end{split}
\end{equation*}
where $(X_{i_{k-1}, t_{k-1}},U_{i_{k-1}, t_{k-1}})$ are taken from the data set $D_{j^{l+1:\rho},s^{l+1:\rho}}$.

The variations of Proposition \ref{prop:asymptotic} needed to prove Proposition \ref{prop:vanishing_weights} will be contained in the following lemma:
\begin{lem}\label{lem:asymptotic}
Let Assumption 1.1 be satisfied. The following results hold for all $\tau\leq T$.
\begin{enumerate}[leftmargin=*]
	\item For any multi-index for the data $j^{\rho},s^{\rho}$ where $\rho\in\mathbb{N}$, it holds that
	\begin{equation*}
		\mathbb{E}[V_{n,i^{\tau}, t^{\tau}}(D_{j^{\rho},s^{\rho}})]\leq (n-\tau-1)^{-|i^{\tau}|}K^{-(\tau+1-|i^{\tau}|)}_n.
	\end{equation*}
	 \item For $r=1,\dots,\tau$ and $i_0=i_r$, it holds that 
	 \begin{align*}
	 	\mathbb{E}&[ |V_{n,i_0, t_0}(\tilde{X}_1,u,D_{i^{r:\tau},t^{r:\tau}})-V_{n,i_0, t_0}(\tilde{X}_1,u,D_{i^{r+1:\tau},t^{r+1:\tau}})|V_{n,i^{0:\tau}, t^{0:\tau}}(D_{i^{r:\tau},t^{r:\tau}})]\\
	 	\leq&
	 	2(n-\tau-1)^{-|i^{\tau}|}K^{-(\tau+1-|i^{\tau}|)}_n.
	 \end{align*}
 \item For $r=0,\dots,\tau$; $l=0,\dots,\tau-1$ and $i_{l+1}= i_r$,
 \begin{align*}
 	\mathbb{E}&[ V_{n,i^{l}, t^{l}}(D_{i^{r+1:\tau},t^{r+1:\tau}})\\
 	&\quad\cdot|\tilde{V}_{n,[(i_{l}, t_{l}),(i_{l+1}, t_{l+1})]}(D_{j^{r+1:\tau},s^{r+1:\tau}}, D_{j^{r:\tau},s^{r:\tau}})-V_{n,[(i_{l}, t_{l}),(i_{l+1}, t_{l+1})]}( D_{j^{r:\tau},s^{r:\tau}})|\\
 	&\quad\cdot V_{n,i^{l+1:\tau}, t^{l+1:\tau}}(D_{i^{r:\tau},t^{r:\tau}})]\\
 	\leq&
 	2(n-\tau-1)^{-|i^{\tau}|}K^{-(\tau+1-|i^{\tau}|)}_n.
 \end{align*}
 \item For $r=0,\dots,\tau$; $l=0,\dots,\tau-1$ and $i_{l}= i_r$,
\begin{align*}
	\mathbb{E}&[ V_{n,i^{l}, t^{l}}(D_{i^{r+1:\tau},t^{r+1:\tau}})\\
	&\quad\cdot|V_{n,[(i_{l}, t_{l}),(i_{l+1}, t_{l+1})]}( D_{j^{r+1:\tau},s^{r+1:\tau}})\\
	&\qquad-\tilde{V}_{n,[(i_{l}, t_{l}),(i_{l+1}, t_{l+1})]}(D_{j^{r+1:\tau},s^{r+1:\tau}}, D_{j^{r:\tau},s^{r:\tau}})|\\
	&\quad\cdot V_{n,i^{l+1:\tau}, t^{l+1:\tau}}(D_{i^{r:\tau},t^{r:\tau}})]\\
	\leq&
	2(n-\tau-1)^{-|i^{\tau}|}K^{-(\tau+1-|i^{\tau}|)}_n.
\end{align*}
\item For $r=0,\dots,\tau$ and $l=0,\dots,\tau-1$, where either $l=r-1$ and $i_l\neq i_r$ or $i_l\neq i_r$ and $i_{l+1}\neq i_r$, it holds that
\begin{align*}
	\mathbb{E}&[ V_{n,i^{l}, t^{l}}(D_{i^{r+1:\tau},t^{r+1:\tau}})\\
	&\quad\cdot|V_{n,[(i_{l}, t_{l}),(i_{l+1}, t_{l+1})]}( D_{i^{r+1:\tau},t^{r+1:\tau}})-V_{n,[(i_{l}, t_{l}),(i_{l+1}, t_{l+1})]}( D_{i^{r:\tau},t^{r:\tau}})|\\
	&\quad\cdot V_{n,i^{l+1:\tau}, t^{l+1:\tau}}(D_{i^{r:\tau},t^{r:\tau}})]\\
	\leq&
	\begin{cases}
		 2(n-\tau-1)^{-|i^{\tau}|}K^{-(\tau+1-|i^{\tau}|)}_n & \text{ if } |i^l|=|i^{l+1}|\\
		 4(T+1)(n-\tau-1)^{-|i^{\tau}|}K^{-(\tau+2-|i^{\tau}|)}_n &\text{ if } |i^l|+1=|i^{l+1}|
	\end{cases}    
\end{align*}
and if $r=0$ or $i_0\neq i_r$, then 
\begin{align*}
	\mathbb{E}&[ |V_{n,i_0, t_0}(\tilde{X}_1,u,D_{i^{r:\tau},t^{r:\tau}})-V_{n,i_0, t_0}(\tilde{X}_1,u,D_{i^{r+1:\tau},t^{r+1:\tau}})|\\
	&\quad\cdot V_{n,i^{0:\tau}, t^{0:\tau}}(D_{i^{r:\tau},t^{r:\tau}})]\\
	\leq& 2(T+1)(n-\tau-1)^{-|i^{\tau}|}K^{-(\tau+2-|i^{\tau}|)}_n    
\end{align*}
\item For $r=0,\dots,\tau$ and $l=0,\dots,\tau-1$, it holds that
\begin{align*}
	\mathbb{E}&[ V_{n,i^{l}, t^{l}}(D_{i^{r+1:\tau},t^{r+1:\tau}})\\
	&\quad\cdot|V_{n,[(i_{l}, t_{l}),(i_{l+1}, t_{l+1})]}( D_{j^{r+1:\tau},s^{r+1:\tau}})-V_{n,[(i_{l}, t_{l}),(i_{l+1}, t_{l+1})]}( D_{j^{r:\tau},s^{r:\tau}})|\\
	&\quad\cdot V_{n,i^{l+1:\tau}, t^{l+1:\tau}}(D_{i^{r:\tau},t^{r:\tau}})]\\
	\leq&
	\begin{cases}
		4(n-\tau-1)^{-|i^{\tau}|}K^{-(\tau+1-|i^{\tau}|)}_n & \text{ if } |i^\tau|<\tau+1\\
		16(T+1)(n-\tau-1)^{-|i^{\tau}|}K^{-(\tau+2-|i^{\tau}|)}_n &\text{ if } |i^\tau|=\tau+1
	\end{cases}    
\end{align*}
and 
\begin{align*}
	\mathbb{E}&[ |V_{n,i_0, t_0}(\tilde{X}_1,u,D_{i^{r:\tau},t^{r:\tau}})-V_{n,i_0, t_0}(\tilde{X}_1,u,D_{i^{r+1:\tau},t^{r+1:\tau}})|\\
	&\quad\cdot V_{n,i^{0:\tau}, t^{0:\tau}}(D_{i^{r:\tau},t^{r:\tau}})]\\
	\leq&
	\begin{cases}
		4(n-\tau-1)^{-|i^{\tau}|}K^{-(\tau+1-|i^{\tau}|)}_n & \text{ if } |i^\tau|<\tau+1\\
		2T(n-\tau-1)^{-|i^{\tau}|}K^{-(\tau+2-|i^{\tau}|)}_n &\text{ if } |i^\tau|=\tau+1
	\end{cases}    
\end{align*}
\item Consider two multi-indices $i^{\tau},t^{\tau}$ and $j^{\tau},s^{\tau}$ and two data sets with a finite number of replacements $\tilde{D}_1$ and $\tilde{D}_2$ compared to $D$, then 
\begin{equation*}
	\mathbb{E}[V_{n,i^{\tau}, t^{\tau}}(\tilde{D}_1)V_{n,j^{\tau}, s^{\tau}}(\tilde{D}_2)]\leq (n-2\tau-2)^{-|(i^{\tau}, j^{\tau})|}K^{-(2\tau+2-|(i^{\tau}, j^{\tau})|)}_n.
\end{equation*}
 \item  Consider two multi-indices $i^{\tau},t^{\tau}$ and $j^{\tau},s^{\tau}$. For $r=0,\dots,\tau$ and $l=0,\dots,\tau-1$, it holds that
 \begin{align*}
 	\mathbb{E}&[ V_{n,i^{l}, t^{l}}(D_{i^{r+1:\tau},t^{r+1:\tau}})\\
 	&\quad\cdot|V_{n,[(i_{l}, t_{l}),(i_{l+1}, t_{l+1})]}( D_{i^{r+1:\tau},t^{r+1:\tau}})-V_{n,[(i_{l}, t_{l}),(i_{l+1}, t_{l+1})]}( D_{i^{r:\tau},t^{r:\tau}})|\\
 	&\quad\cdot V_{n,i^{l+1:\tau}, t^{l+1:\tau}}(D_{i^{r:\tau},t^{r:\tau}})V_{n,j^{l}, s^{l}}(D_{j^{r+1:\tau},s^{r+1:\tau}})\\
 	&\quad\cdot|V_{n,[(j_{l}, s_{l}),(j_{l+1}, s_{l+1})]}( D_{j^{r+1:\tau},s^{r+1:\tau}})-V_{n,[(j_{l}, s_{l}),(j_{l+1}, s_{l+1})]}( D_{j^{r:\tau},s^{r:\tau}})|\\
 	&\quad\cdot V_{n,j^{l+1:\tau}, s^{l+1:\tau}}(D_{j^{r:\tau},s^{r:\tau}})]\\
 	\leq&
 	\begin{cases}
 		8(n-2\tau-2)^{-|(i^{\tau}, j^{\tau})|}K^{-(\tau+1-|(i^{\tau}, j^{\tau})|)}_n & \text{ if } |(i^{\tau}, j^{\tau})|<2\tau+2\\
 		4T^2(n-2\tau-2)^{-|i^{\tau}|}K^{-(\tau+3-|i^{\tau}|)}_n &\text{ if } |(i^{\tau}, j^{\tau})|=2\tau+2
 	\end{cases}    
 \end{align*}
 and 
\begin{align*}
 	\mathbb{E}&[ |V_{n,i_0, t_0}(\tilde{X}_1,u,D_{i^{r:\tau},t^{r:\tau}})-V_{n,i_0, t_0}(\tilde{X}_1,u,D_{i^{r+1:\tau},t^{r+1:\tau}})|\\
 	&\quad\cdot V_{n,i^{0:\tau}, t^{0:\tau}}(D_{i^{r:\tau},t^{r:\tau}})|V_{n,j_0, s_0}(\tilde{X}_1,u,D_{j^{r:\tau},s^{r:\tau}})-V_{n,j_0, s_0}(\tilde{X}_1,u,D_{j^{r+1:\tau},s^{r+1:\tau}})|\\
 	&\quad\cdot V_{n,j^{0:\tau}, s^{0:\tau}}(D_{j^{r:\tau},s^{r:\tau}})]\\
 	\leq&
 	\begin{cases}
 		8(n-2\tau-2)^{-|(i^{\tau}, j^{\tau})|}K^{-(\tau+1-|(i^{\tau}, j^{\tau})|)}_n & \text{ if } |(i^{\tau}, j^{\tau})|<2\tau+2\\
 		4(T+1)^2(n-2\tau-2)^{-|i^{\tau}|}K^{-(\tau+3-|i^{\tau}|)}_n &\text{ if } |(i^{\tau}, j^{\tau})|=2\tau+2
 	\end{cases}    
 \end{align*} 
\item Consider two multi-indices $i^{\tau},t^{\tau}$ and $j^{\tau},s^{\tau}$ where $|(i^{\tau},j^{\tau})|=2\tau+2$.  For $r=0,\dots,\tau$ and $l=0,\dots,\tau-1$, it holds that
\begin{align*}
	\mathbb{E}&[ V_{n,i^{\tau}, t^{\tau}}(D_{i^{\tau},t^{\tau}})V_{n,j^{l}, s^{l}}(D_{(i^{r+1:\tau},j^{\tau}),(t^{r+1:\tau},s^{\tau})})\\
	&\quad\cdot|V_{n,[(j_{l}, s_{l}),(j_{l+1}, s_{l+1})]} (D_{(i^{r+1:\tau},j^{\tau}),(t^{r+1:\tau},s^{\tau})})\\
	&\qquad-V_{n,[(j_{l}, s_{l}),(j_{l+1}, s_{l+1})]}(D_{(i^{r:\tau},j^{\tau}),(t^{r:\tau},s^{\tau})})|\\
	&\quad\cdot V_{n,j^{l+1:\tau}, s^{l+1:\tau}}(D_{(i^{r:\tau},j^{\tau}),(t^{r:\tau},s^{\tau})})]\\
	\leq&2T(n-2\tau-2)^{-(2\tau+2)}K^{-1}_n      
\end{align*}
and 
\begin{align*}
	\mathbb{E}&[ V_{n,i^{\tau}, t^{\tau}}(D_{i^{\tau},t^{\tau}})\\
	&\quad\cdot|V_{n,j_0, s_0}(\tilde{X}_1,u,D_{(i^{r+1:\tau},j^{\tau}),(t^{r+1:\tau},s^{\tau})})-V_{n,j_0, s_0}(\tilde{X}_1,u,D_{(i^{r:\tau},j^{\tau}),(t^{r:\tau},s^{\tau})})|\\
	&\quad\cdot V_{n,j^{0:\tau}, s^{0:\tau}}(D_{(i^{r:\tau},j^{\tau}),(t^{r:\tau},s^{\tau})})]\\
	\leq&2T(n-2\tau-2)^{-(2\tau+2)}K^{-1}_n   
\end{align*} 	 
\end{enumerate}
\end{lem}
\begin{proof}
\begin{enumerate}[leftmargin=*]
	\item The proof is completely analogous to Proposition \ref{prop:asymptotic}.
	\item Using the triangle inequality 
	\begin{align*}
	\mathbb{E}&[ |V_{n,i_0, t_0}(\tilde{X}_1,u,D_{i^{r:\tau},t^{r:\tau}})-V_{n,i_0, t_0}(\tilde{X}_1,u,D_{i^{r+1:\tau},t^{r+1:\tau}})|V_{n,i^{0:\tau}, t^{0:\tau}}(D_{i^{r:\tau},t^{r:\tau}})]\\
	\leq& \mathbb{E}[ V_{n,i_0, t_0}(\tilde{X}_1,u,D_{i^{r:\tau},t^{r:\tau}})V_{n,i^{0:\tau}, t^{0:\tau}}(D_{i^{r:\tau},t^{r:\tau}})]\\
	&\qquad+\mathbb{E}[V_{n,i_0, t_0}(\tilde{X}_1,u,D_{i^{r+1:\tau},t^{r+1:\tau}})V_{n,i^{0:\tau}, t^{0:\tau}}(D_{i^{r:\tau},t^{r:\tau}})]
	\end{align*}
	The two terms can be treated as in Proposition \ref{prop:asymptotic}.
	\item The arguments are analogous to item 2. 
	\item The arguments are analogous to item 2.  
	\item We will only discuss the first upper-bound. The second is almost completely analogous. 
	
	Consider the case where $|i^l|=|i^{l+1}|$. For this use the triangle inequality as in the proof to Lemma \ref{lem:asymptotic}.2 and then use the argument as in the induction step in the proof of Proposition \ref{prop:asymptotic}.
	
	If $|i^l|+1=|i^{l+1}|$, note that almost surely
	\begin{align*}
	|&V_{n,[(i_{l}, t_{l}),(i_{l+1}, t_{l+1})]}( D_{i^{r+1:\tau},t^{r+1:\tau}})-V_{n,[(i_{l}, t_{l}),(i_{l+1}, t_{l+1})]}( D_{i^{r:\tau},t^{r:\tau}})|\\
	&=K_n^{-1}(I_{A_{(i_{l}, t_{l})}^{(i_{l+1}, t_{l+1})}}+I_{B_{(i_{l}, t_{l})}^{(i_{l+1}, t_{l+1})}})
	\end{align*}
	where 
	\begin{align*}
	&A_{(i_{l}, t_{l})}^{(i_{l+1}, t_{l+1})}\coloneqq \Big\{V_{n,[(i_{l}, t_{l}),(i_{l+1}, t_{l+1})]}( D_{i^{r+1:\tau},t^{r+1:\tau}})=1\\
	&\qquad\text{ and }V_{n,[(i_{l}, t_{l}),(i_{l+1}, t_{l+1})]}( D_{i^{r:\tau},t^{r:\tau}})=0\Big\}\\
	&B_{(i_{l}, t_{l})}^{(i_{l+1}, t_{l+1})}\coloneqq \Big\{V_{n,[(i_{l}, t_{l}),(i_{l+1}, t_{l+1})]}( D_{i^{r+1:\tau},t^{r+1:\tau}})=0\\
	&\qquad\text{ and }V_{n,[(i_{l}, t_{l}),(i_{l+1}, t_{l+1})]}( D_{i^{r:\tau},t^{r:\tau}})=1\Big\}.
	\end{align*}
	For a data set $D$ and two integers $a,b\in\mathbb{N}$ with $a<b$, define the set
	\begin{align*}
		&C_{(i_{l}, t_{l})}^{(i_{l+1}, t_{l+1})}(D,a,b)\\
		&\quad\coloneqq
		\Big\{(X_{i_{l+1}, t_{l+1}},U_{i_{l+1}, t_{l+1}})\text{ is not a } (a-1)\text{-NN but a }b\text{-NN in }D\Big\}. 
	\end{align*}
	 
	Note that $(X_{i_{l}, t_{l}+1},U_{i_{l}, t_{l}+1})$ and $(X_{i_{l+1}, t_{l+1}},U_{i_{l+1}, t_{l+1}})$ are the same random variables in $ D_{i^{r:\tau},t^{r:\tau}}$ and in $ D_{i^{r+1:\tau},t^{r+1:\tau}}$. Since at most $T$ elements are exchanged from  $ D_{i^{r+1:\tau},t^{r+1:\tau}}$ compared to $ D_{i^{r:\tau},t^{r:\tau}}$, the nearest neighbor ranking will at most change by $2T$ ranks. Thus, it holds that:
	\begin{align*}
		&A_{(i_{l}, t_{l})}^{(i_{l+1}, t_{l+1})}\subset C_{(i_{l}, t_{l})}^{(i_{l+1}, t_{l+1})}( D_{i^{r+1:\tau},t^{r+1:\tau}},K_n-2T,K_n)\\
		&B_{(i_{l}, t_{l})}^{(i_{l+1}, t_{l+1})}\subset C_{(i_{l}, t_{l})}^{(i_{l+1}, t_{l+1})}( D_{i^{r+1:\tau},t^{r+1:\tau}},K_n+1,K_n+2T+1)
	\end{align*}
	and 
		\begin{align*}
		\mathbb{E}&\Big[V_{n,i^{l}, t^{l}}(D_{i^{r+1:\tau},t^{r+1:\tau}})
		|V_{n,[(i_{l}, t_{l}),(i_{l+1}, t_{l+1})]}( D_{i^{r+1:\tau},t^{r+1:\tau}})\\
		&\qquad-V_{n,[(i_{l}, t_{l}),(i_{l+1}, t_{l+1})]}( D_{i^{r:\tau},t^{r:\tau}})|\Big]\\
		&\leq K_n^{-1}\mathbb{E}[V_{n,i^{l}, t^{l}}(D_{i^{r+1:\tau},t^{r+1:\tau}})
		I_{C_{(i_{l}, t_{l})}^{(i_{l+1}, t_{l+1})}( D_{i^{r+1:\tau},t^{r+1:\tau}},K_n-2T,K_n)}]\\
		&\quad+ K_n^{-1}\mathbb{E}[V_{n,i^{l}, t^{l}}(D_{i^{r+1:\tau},t^{r+1:\tau}})
		I_{C_{(i_{l}, t_{l})}^{(i_{l+1}, t_{l+1})}( D_{i^{r+1:\tau},t^{r+1:\tau}},K_n+1,K_n+2T+1)}]
	\end{align*}
	We will prove the appropriate upper-bound for the first term. The upper-bound for the second term can be proved analogously. 
	
	Assume that 
	\begin{align*}
		&K_n^{-1}\mathbb{E}[V_{n,i^{l}, t^{l}}(D_{i^{r+1:\tau},t^{r+1:\tau}})
		I_{C_{(i_{l}, t_{l})}^{(i_{l+1}, t_{l+1})}( D_{i^{r+1:\tau},t^{r+1:\tau}},K_n-2T,K_n)}]\\
		&\quad>2T(n-l)^{-|i^{l+1}|}K^{|i^{l+1}|-l-3}_n 
	\end{align*}
	By the iid assumption across episodes it has to hold that 
	\begin{align*}
		&K_n^{-1}\mathbb{E}[V_{n,i^{l}, t^{l}}(D_{i^{r+1:\tau},t^{r+1:\tau}})
		I_{C_{(i_{l}, t_{l})}^{(j, t_{l+1})}( D_{i^{r+1:\tau},t^{r+1:\tau}},K_n-2T,K_n)}]\\
		&\quad>2T(n-l)^{-|i^{l+1}|}K^{|i^{l+1}|-l-3}_n 
	\end{align*}
	for all $j$ where $|(i^l,j)|=|i^l|+1$. With Lemma \ref{lem:asymptotic}.1, we then have that 
	\begin{align*}
		\mathbb{E}&[
		I_{C_{(i_{l}, t_{l})}^{(j, t_{l+1})}( D_{i^{r+1:\tau},t^{r+1:\tau}},K_n-2T,K_n)}|V_{n,i^{l}, t^{l}}(D_{i^{r+1:\tau},t^{r+1:\tau}})=K_n^{-(l+1)}]\\
		&>2T(n-l)^{-1} .
	\end{align*}
	Note that by definition it has to hold 
	
	\begin{align*}
		\frac{1}{2T}\sum_{i=1}^n\sum_{t=0}^{T-1}\mathbb{E}[
		I_{C_{(i_{l}, t_{l})}^{(i, t)}( D_{i^{r+1:\tau},t^{r+1:\tau}},K_n-2T,K_n)}|V_{n,i^{l}, t^{l}}(D_{i^{r+1:\tau},t^{r+1:\tau}})]=1.
	\end{align*}
	Splitting the sum in a similar fashion as in the proof of Proposition \ref{prop:asymptotic} then yields
	\begin{align*}
	1&\geq \frac{1}{2T}\sum_{\substack{i=1\\
	i\neq i^l}}^n\sum_{t=0}^{T-1}\mathbb{E}[
	I_{C_{(i_{l}, t_{l})}^{(i, t_{l+1})}( D_{i^{r+1:\tau},t^{r+1:\tau}},K_n-2T,K_n)}|V_{n,i^{l}, t^{l}}(D_{i^{r+1:\tau},t^{r+1:\tau}})]\\
	&>\frac{2T(n-l)}{2T(n-l)}=1
	\end{align*}
	This is a contradiction and, hence, 
	\begin{align*}
		&K_n^{-1}\mathbb{E}[V_{n,i^{l}, t^{l}}(D_{i^{r+1:\tau},t^{r+1:\tau}})
		I_{C_{(i_{l}, t_{l})}^{(i_{l+1}, t_{l+1})}( D_{i^{r+1:\tau},t^{r+1:\tau}},K_n-2T,K_n)}]\\
		&\quad\leq2T(n-l)^{-|i^{l+1}|}K^{|i^{l+1}|-l-3}_n 
	\end{align*}
	Using the same arguments as in the induction steps in the proof of Proposition \ref{prop:asymptotic} then yields the full result.
	\item This result is a direct consequence from the previous four results. 
	\item Use the asymptotic behaviour known from Lemma \ref{lem:asymptotic}.1	$\mathbb{E}[V_{n,i^{\tau}, t^{\tau}}(\tilde{D}_1)]$. Then use the same arguments as in the induction step in Proposition \ref{prop:asymptotic} for $\mathbb{E}[V_{n,i^{\tau}, t^{\tau}}(\tilde{D}_1)V_{n,j_0, s_0}(\tilde{X}_1,u,(\tilde{D}_2))]$ and for each  $\mathbb{E}[V_{n,i^{\tau}, t^{\tau}}(\tilde{D}_1)V_{n,j^k, s^k}(\tilde{X}_1,u,(\tilde{D}_2))]$ where $k<\tau$. 
	\item This result is a direct combination of the arguments used in Lemma \ref{lem:asymptotic}.6 and \ref{lem:asymptotic}.7.
	\item This result is again a direct combination of the arguments used in Lemma \ref{lem:asymptotic}.6 and \ref{lem:asymptotic}.7.
\end{enumerate}	
\end{proof}
\begin{rem}
Note that all the results in this section and all the following results can be easily extended to a version where the weights prevent reusing the same sample twice in a $K_n$-NN path by simply observing that the weights are zero in such cases. This corresponds to the implemented algorithm in some of the experiments (see Appendix Section \ref{sec:impl}). 
\end{rem}

\begin{proof}[Proof of Proposition  \ref{prop:vanishing_weights}]
\begin{enumerate}[leftmargin=*]
	\item By Lemma \ref{lem:asymptotic}.7 it holds that
	\begin{align*}
		\mathbb{E}&\Bigg[\sum_{i^\tau,t^\tau}^{n,T}\sum_{\substack{j^\tau,s^\tau\\  |(i^\tau,j^\tau)|<2\tau+2}}^{n,T} V_{n,i^{\tau}, t^{\tau}}(\tilde{D}_1) V_{n,j^\tau,s^\tau}(\tilde{D}_2) \Bigg]\\
		\leq& \sum_{i^\tau,t^\tau}^{n,T}\sum_{\substack{j^\tau,s^\tau\\  |(i^\tau,j^\tau)|<2\tau+2}}^{n,T} (n-2\tau-2)^{-|(i^{\tau}, j^{\tau})|}K^{-(2\tau+2-|(i^{\tau}, j^{\tau})|)}_n\\
		=& T^{2\tau} \sum_{i^\tau}^{n}\sum_{\substack{j^\tau\\  |(i^\tau,j^\tau)|<2\tau+2}}^{n}(n-2\tau-2)^{-|(i^{\tau}, j^{\tau})|}K^{-(2\tau+2-|(i^{\tau}, j^{\tau})|)}_n
	\end{align*}
	where $\sum_{i^k}^{n}\coloneqq \sum_{i_0=1}^{n}\cdots\sum_{i_{k}=1}^{n}$. Fix an arbitrary $a\in\mathbb{N}$ with $a<2\tau+2$, then it holds that 
	\begin{align*}
		T^{2\tau} \sum_{i^\tau}^{n}\sum_{\substack{j^\tau\\  |(i^\tau,j^\tau)|=a}}^{n} (n-2\tau-2)^{-a}K^{-(2\tau+2-a)}_n=\frac{n^a}{(n-2\tau-2)^{a}K^{2\tau+2-a}_n}\rightarrow0
	\end{align*}
	since $K_n\rightarrow\infty$.
	Since $a$ was arbitrary the statement follows.
	\item Note that
	\begin{align*}
		\mathbb{E}&\Bigg[\Bigg(\sum_{i^{\tau}, t^{\tau}}^{n, T} |V_{n,i^{\tau}, t^{\tau}}(D_{i^{\tau},t^{\tau}})-V_{n,i^{\tau}, t^{\tau}}(D)|\Bigg)^2\Bigg]\\
		\leq&\mathbb{E}\Bigg[\Bigg(\sum_{i^{\tau}, t^{\tau}}^{n, T} \sum_{r=0}^{\tau}|V_{n,i^{\tau}, t^{\tau}}(D_{i^{r:\tau},t^{r:\tau}})-V_{n,i^{\tau}, t^{\tau}}(D_{i^{r+1:\tau},t^{r+1:\tau}})|\Bigg)^2\Bigg]\\
		\leq&(\tau+1)\mathbb{E}\Bigg[\sum_{r=0}^{\tau}\left(\sum_{i^{\tau}, t^{\tau}}^{n, T}|V_{n,i^{\tau}, t^{\tau}}(D_{i^{r:\tau},t^{r:\tau}})-V_{n,i^{\tau}, t^{\tau}}(D_{i^{r+1:\tau},t^{r+1:\tau}})\right)^2\Bigg]
	\end{align*}
	where we define $D_{i^{\tau+1:\tau},t^{\tau+1:\tau}}\coloneqq D$. Furthermore, for any $r=0,\dots,\tau$, we have
	\begin{align*}
		\mathbb{E}&\Bigg[\Bigg(\sum_{i^{\tau}, t^{\tau}}^{n, T}|V_{n,i^{\tau}, t^{\tau}}(D_{i^{r:\tau},t^{r:\tau}})-V_{n,i^{\tau}, t^{\tau}}(D_{i^{r+1:\tau},t^{r+1:\tau}})\Bigg)^2\Bigg]\\
		\leq& \mathbb{E}\Bigg[\Bigg(\sum_{i^{\tau}, t^{\tau}}^{n, T}
		|V_{n,i_0, t_0}(\tilde{X}_1,u,D_{i^{r:\tau},t^{r:\tau}})-V_{n,i_0, t_0}(\tilde{X}_1,u,D_{i^{r+1:\tau},t^{r+1:\tau}})|\\
		&\quad\cdot V_{n,i^{0:\tau}, t^{0:\tau}}(D_{i^{r:\tau},t^{r:\tau}})\\
		&+\sum_{l=0}^{\tau-1} V_{n,i^{l}, t^{l}}(D_{i^{r+1:\tau},t^{r+1:\tau}})\\
		&\quad\cdot|V_{n,[(i_{l}, t_{l}),(i_{l+1}, t_{l+1})]}( D_{j^{r+1:\tau},s^{r+1:\tau}})-V_{n,[(i_{l}, t_{l}),(i_{l+1}, t_{l+1})]}( D_{j^{r:\tau},s^{r:\tau}})|\\
		&\quad\cdot V_{n,i^{l+1:\tau}, t^{l+1:\tau}}(D_{i^{r:\tau},t^{r:\tau}})\Bigg)^2\Bigg]\\
		\leq& (\tau+1)\mathbb{E}\Big[\Big(\sum_{i^{\tau}, t^{\tau}}^{n, T}
		|V_{n,i_0, t_0}(\tilde{X}_1,u,D_{i^{r:\tau},t^{r:\tau}})-V_{n,i_0, t_0}(\tilde{X}_1,u,D_{i^{r+1:\tau},t^{r+1:\tau}})|\\
		&\quad\cdot V_{n,i^{0:\tau}, t^{0:\tau}}(D_{i^{r:\tau},t^{r:\tau}})\Big)^2\Big]\\
		&+(\tau+1)\sum_{l=0}^{\tau-1}\mathbb{E}\Big[\Big(\sum_{i^{\tau}, t^{\tau}}^{n, T} V_{n,i^{l}, t^{l}}(D_{i^{r+1:\tau},t^{r+1:\tau}})\\
		&\quad\cdot|V_{n,[(i_{l}, t_{l}),(i_{l+1}, t_{l+1})]}( D_{j^{r+1:\tau},s^{r+1:\tau}})-V_{n,[(i_{l}, t_{l}),(i_{l+1}, t_{l+1})]}( D_{j^{r:\tau},s^{r:\tau}})|\\
		&\quad\cdot V_{n,i^{l+1:\tau}, t^{l+1:\tau}}(D_{i^{r:\tau},t^{r:\tau}})\Big)^2\Big]
	\end{align*}
	Fix an $l=0,\dots,\tau-1$,
	\begin{align*}
		\mathbb{E}&\Big[\Big(\sum_{i^{\tau}, t^{\tau}}^{n, T} V_{n,i^{l}, t^{l}}(D_{i^{r+1:\tau},t^{r+1:\tau}})\\
		&\quad\cdot|V_{n,[(i_{l}, t_{l}),(i_{l+1}, t_{l+1})]}( D_{j^{r+1:\tau},s^{r+1:\tau}})-V_{n,[(i_{l}, t_{l}),(i_{l+1}, t_{l+1})]}( D_{j^{r:\tau},s^{r:\tau}})|\\
		&\quad\cdot V_{n,i^{l+1:\tau}, t^{l+1:\tau}}(D_{i^{r:\tau},t^{r:\tau}})\Big)^2\Big]\\
		=& \mathbb{E}\Big[\sum_{i^{\tau}, t^{\tau}}^{n, T}\sum_{j^{\tau}, s^{\tau}}^{n, T} V_{n,i^{l}, t^{l}}(D_{i^{r+1:\tau},t^{r+1:\tau}})\\
		&\quad\cdot|V_{n,[(i_{l}, t_{l}),(i_{l+1}, t_{l+1})]}( D_{j^{r+1:\tau},s^{r+1:\tau}})-V_{n,[(i_{l}, t_{l}),(i_{l+1}, t_{l+1})]}( D_{j^{r:\tau},s^{r:\tau}})|\\
		&\quad\cdot V_{n,i^{l+1:\tau}, t^{l+1:\tau}}(D_{i^{r:\tau},t^{r:\tau}})V_{n,j^{l}, s^{l}}(D_{j^{r+1:\tau},s^{r+1:\tau}})\\
		&\quad\cdot|V_{n,[(j_{l}, s_{l}),(j_{l+1}, s_{l+1})]}( D_{j^{r+1:\tau},s^{r+1:\tau}})-V_{n,[(j_{l}, s_{l}),(j_{l+1}, s_{l+1})]}( D_{j^{r:\tau},s^{r:\tau}})|\\
		&\quad\cdot V_{n,j^{l+1:\tau}, s^{l+1:\tau}}(D_{j^{r:\tau},s^{r:\tau}})\Big]
	\end{align*}
	For any $a\in\mathbb{N}$ with $a<2\tau+2$, Lemma\ref{lem:asymptotic}.8 implies
	\begin{align*}
		\mathbb{E}&\Big[\sum_{i^{\tau}, t^{\tau}}^{n, T}\sum_{\substack{j^{\tau}, s^{\tau}\\|j^{\tau}, i^{\tau}|=a}}^{n, T} V_{n,i^{l}, t^{l}}(D_{i^{r+1:\tau},t^{r+1:\tau}})\\
		&\quad\cdot|V_{n,[(i_{l}, t_{l}),(i_{l+1}, t_{l+1})]}( D_{j^{r+1:\tau},s^{r+1:\tau}})-V_{n,[(i_{l}, t_{l}),(i_{l+1}, t_{l+1})]}( D_{j^{r:\tau},s^{r:\tau}})|\\
		&\quad\cdot V_{n,i^{l+1:\tau}, t^{l+1:\tau}}(D_{i^{r:\tau},t^{r:\tau}})V_{n,j^{l}, s^{l}}(D_{j^{r+1:\tau},s^{r+1:\tau}})\\
		&\quad\cdot|V_{n,[(j_{l}, s_{l}),(j_{l+1}, s_{l+1})]}( D_{j^{r+1:\tau},s^{r+1:\tau}})-V_{n,[(j_{l}, s_{l}),(j_{l+1}, s_{l+1})]}( D_{j^{r:\tau},s^{r:\tau}})|\\
		&\quad\cdot V_{n,j^{l+1:\tau}, s^{l+1:\tau}}(D_{j^{r:\tau},s^{r:\tau}})\Big]\\
		&\leq L\frac{n^a}{(n-2\tau-2)^{a}K^{(\tau+1-a)}_n }\rightarrow0
	\end{align*}
	and 
	\begin{align*}
		\mathbb{E}&\Big[\sum_{i^{\tau}, t^{\tau}}^{n, T}\sum_{\substack{j^{\tau}, s^{\tau}\\|j^{\tau}, i^{\tau}|=2\tau+2}}^{n, T} V_{n,i^{l}, t^{l}}(D_{i^{r+1:\tau},t^{r+1:\tau}})\\
		&\quad\cdot|V_{n,[(i_{l}, t_{l}),(i_{l+1}, t_{l+1})]}( D_{j^{r+1:\tau},s^{r+1:\tau}})-V_{n,[(i_{l}, t_{l}),(i_{l+1}, t_{l+1})]}( D_{j^{r:\tau},s^{r:\tau}})|\\
		&\quad\cdot V_{n,i^{l+1:\tau}, t^{l+1:\tau}}(D_{i^{r:\tau},t^{r:\tau}})V_{n,j^{l}, s^{l}}(D_{j^{r+1:\tau},s^{r+1:\tau}})\\
		&\quad\cdot|V_{n,[(j_{l}, s_{l}),(j_{l+1}, s_{l+1})]}( D_{j^{r+1:\tau},s^{r+1:\tau}})-V_{n,[(j_{l}, s_{l}),(j_{l+1}, s_{l+1})]}( D_{j^{r:\tau},s^{r:\tau}})|\\
		&\quad\cdot V_{n,j^{l+1:\tau}, s^{l+1:\tau}}(D_{j^{r:\tau},s^{r:\tau}})\Big]\\
		&\leq \tilde{L}\frac{n^{2\tau+2}}{(n-2\tau-2)^{2\tau+2}K^2_n }\rightarrow0
	\end{align*}
	where $L,\tilde{L}>0$  are some constants.
	Note that the analogous arguments can be applied to
	\begin{align*}
		\mathbb{E}&\Big[\Big(\sum_{i^{\tau}, t^{\tau}}^{n, T}
		|V_{n,i_0, t_0}(\tilde{X}_1,u,D_{i^{r:\tau},t^{r:\tau}})-V_{n,i_0, t_0}(\tilde{X}_1,u,D_{i^{r+1:\tau},t^{r+1:\tau}})|\\
		&\quad\cdot V_{n,i^{0:\tau}, t^{0:\tau}}(D_{i^{r:\tau},t^{r:\tau}})\Big)^2\Big].
	\end{align*}
\item Note that 
\begin{align*}
	\mathbb{E}&\Bigg[\sum_{i^{\tau}, t^{\tau}}^{n, T} \sum_{\substack{j^\tau,s^\tau\\  |(i^\tau,j^\tau)|=2\tau+2}}^{n,T}
	V_{n,i^{\tau}, t^{\tau}}(D_{i^{\tau},t^{\tau}})
	|V_{n,j^{\tau}, s^{\tau}}(D_{j^{\tau},s^{\tau}})
	-V_{n,j^{\tau}, s^{\tau}}(D_{(i^{\tau},j^{\tau}),(t^{\tau},s^{\tau})})|\Bigg]\\
	\leq&\sum_{r=0}^{\tau} \mathbb{E}\Bigg[\sum_{i^{\tau}, t^{\tau}}^{n, T} \sum_{\substack{j^\tau,s^\tau\\  |(i^\tau,j^\tau)|=2\tau+2}}^{n,T}
	V_{n,i^{\tau}, t^{\tau}}(D_{i^{\tau},t^{\tau}})
	\\
	&\qquad\cdot\Big|
	V_{n,j^{\tau}, s^{\tau}}(D_{(i^{r:\tau},j^{\tau}),(t^{r:\tau},s^{\tau})})
	-V_{n,j^{\tau}, s^{\tau}}(D_{(i^{r+1:\tau},j^{\tau}),(t^{r+1:\tau},s^{\tau})})\Big|\Bigg]
\end{align*}
where $D_{(i^{\tau+1:\tau},j^{\tau}),(t^{\tau+1:\tau},s^{\tau})}\coloneqq D_{j^{\tau},s^{\tau}}$. For any $r=0,\dots,\tau$ it holds that
\begin{align*}
	\mathbb{E}&\Bigg[\sum_{i^{\tau}, t^{\tau}}^{n, T} \sum_{\substack{j^\tau,s^\tau\\  |(i^\tau,j^\tau)|=2\tau+2}}^{n,T}
	V_{n,i^{\tau}, t^{\tau}}(D_{i^{\tau},t^{\tau}})\\
	&\qquad\cdot\Big|
	V_{n,j^{\tau}, s^{\tau}}(D_{(i^{r:\tau},j^{\tau}),(t^{r:\tau},s^{\tau})})
	-V_{n,j^{\tau}, s^{\tau}}(D_{(i^{r+1:\tau},j^{\tau}),(t^{r+1:\tau},s^{\tau})})\Big|\Bigg]\\
	&\leq\sum_{i^{\tau}, t^{\tau}}^{n, T} \sum_{\substack{j^\tau,s^\tau\\  |(i^\tau,j^\tau)|=2\tau+2}}^{n,T} \mathbb{E}[ V_{n,i^{\tau}, t^{\tau}}(D_{i^{\tau},t^{\tau}})\\
	&\quad\cdot|V_{n,j_0, s_0}(\tilde{X}_1,u,D_{(i^{r+1:\tau},j^{\tau}),(t^{r+1:\tau},s^{\tau})})-V_{n,j_0, s_0}(\tilde{X}_1,u,D_{(i^{r:\tau},j^{\tau}),(t^{r:\tau},s^{\tau})})|\\
	&\quad\cdot V_{n,j^{0:\tau}, s^{0:\tau}}(D_{(i^{r:\tau},j^{\tau}),(t^{r:\tau},s^{\tau})})]\\
	&+ \sum_{l=0}^{\tau-1}\sum_{i^{\tau}, t^{\tau}}^{n, T} \sum_{\substack{j^\tau,s^\tau\\  |(i^\tau,j^\tau)|=2\tau+2}}^{n,T}\mathbb{E}[ V_{n,i^{\tau}, t^{\tau}}(D_{i^{\tau},t^{\tau}})V_{n,j^{l}, s^{l}}(D_{(i^{r+1:\tau},j^{\tau}),(t^{r+1:\tau},s^{\tau})})\\
	&\quad\cdot|V_{n,[(j_{l}, s_{l}),(j_{l+1}, s_{l+1})]} (D_{(i^{r+1:\tau},j^{\tau}),(t^{r+1:\tau},s^{\tau})})\\
	&\qquad-V_{n,[(j_{l}, s_{l}),(j_{l+1}, s_{l+1})]}(D_{(i^{r:\tau},j^{\tau}),(t^{r:\tau},s^{\tau})})|\\
	&\quad\cdot V_{n,j^{l+1:\tau}, s^{l+1:\tau}}(D_{(i^{r:\tau},j^{\tau}),(t^{r:\tau},s^{\tau})})]
\end{align*}
For any $l=0,\dots,\tau-1$, Lemma \ref{lem:asymptotic}.9 implies
\begin{align*}
	&\sum_{i^{\tau}, t^{\tau}}^{n, T} \sum_{\substack{j^\tau,s^\tau\\  |(i^\tau,j^\tau)|=2\tau+2}}^{n,T}\mathbb{E}[ V_{n,i^{\tau}, t^{\tau}}(D_{i^{\tau},t^{\tau}})V_{n,j^{l}, s^{l}}(D_{(i^{r+1:\tau},j^{\tau}),(t^{r+1:\tau},s^{\tau})})\\
	&\quad\cdot|V_{n,[(j_{l}, s_{l}),(j_{l+1}, s_{l+1})]} (D_{(i^{r+1:\tau},j^{\tau}),(t^{r+1:\tau},s^{\tau})})\\
	&\qquad-V_{n,[(j_{l}, s_{l}),(j_{l+1}, s_{l+1})]}(D_{(i^{r:\tau},j^{\tau}),(t^{r:\tau},s^{\tau})})|\\
	&\quad\cdot V_{n,j^{l+1:\tau}, s^{l+1:\tau}}(D_{(i^{r:\tau},j^{\tau}),(t^{r:\tau},s^{\tau})})]
	\leq L \frac{n^{2\tau+2}}{(n-2\tau-2)^{2\tau+2}K_n }\rightarrow0
\end{align*}
where $L>0$ is a constant. The arguments for the term
\begin{align*}
	&\sum_{i^{\tau}, t^{\tau}}^{n, T} \sum_{\substack{j^\tau,s^\tau\\  |(i^\tau,j^\tau)|=2\tau+2}}^{n,T} \mathbb{E}[ V_{n,i^{\tau}, t^{\tau}}(D_{i^{\tau},t^{\tau}})\\
	&\quad\cdot|V_{n,j_0, s_0}(\tilde{X}_1,u,D_{(i^{r+1:\tau},j^{\tau}),(t^{r+1:\tau},s^{\tau})})-V_{n,j_0, s_0}(\tilde{X}_1,u,D_{(i^{r:\tau},j^{\tau}),(t^{r:\tau},s^{\tau})})|\\
	&\quad\cdot V_{n,j^{0:\tau}, s^{0:\tau}}(D_{(i^{r:\tau},j^{\tau}),(t^{r:\tau},s^{\tau})})]	
\end{align*}
is analogous.
\end{enumerate}	
\end{proof}

%% file: 20231126_implementation.tex
\section{Implementation}\label{sec:impl}

In this section, we present an extended version of the KNNR Algorithm from the main paper by including a tree-based nearest neighbor search. We then state the MFMC Algorithm from \cite{fonteneau2010model,fonteneau2013batch} and compare the complexity between MFMC and KNNR. Related to this we discuss why tree-based methods and parallelization in trajectory generation is not practical for the MFMC. 

\subsection{Extended KNNR Algorithm}
The main extension of the KNNR algorithm in comparison to the main part is the usage of search trees as a speed-up. Tree-based nearest neighbour search incurs an additional computational cost at the start for fitting a search tree on the given data set. This additional computational cost is more than off-set by a reduced cost for nearest neighbor queries. We will give more details on the complexity of tree-based nearest neighbour search below. For now, we will introduce two generic functions: First, \textsc{TreeBuild}$(\mathcal{D},K)$ takes the entire data set $\mathcal{D}$ and the number of nearest neighbors $K$ as arguments and returns a search tree for $K$-nearest neighbors. The second function \textsc{TreeEval}$((x,v), \mathcal{T})$  uses as arguments a state-action pair $(x,v)\in S$ and a search tree for $K$-nearest neighbors and outputs the $K$-nearest neighbors to $(x,v)$. For actual implementation of the search trees popular choices are K-D Tree \cite{bentley1975multidimensional} and Ball Tree \cite{omohundro1989five}. In our experiments, we use the heuristic provided in \cite{scikit-learn} to choose the nearest neighbor method.

\begin{algorithm}
	\renewcommand{\thealgorithm}{KNNR}
	\caption{ $K$-nearest neighbor resampling for OPE (extended)}\label{alg:matching_ext}
	\hspace*{\algorithmicindent} \textbf{Input:}  Data set $\mathcal{D}$; target policy $u$; nearest neighbors parameter $K$; metric $d$ on $S$; number of resampled trajectories $l$
	\begin{algorithmic}
		\STATE {$R\gets 0$}
		\STATE {$\mathcal{T}\gets$\textsc{TreeBuild}$(\mathcal{D},K)$}
		\FOR{$j$ from $1$ to $l$}
		\STATE{$r\gets 0$}
		\STATE{Sample an initial state $X^{\mathrm{NN}}_{j,0}$ from $\nu_0$}
		\FOR{$s$ from $0$ to $T$ }
		\STATE{Randomly choose $\mathcal{K}_s$ where $\mathcal{K}_s \sim \mathcal{V}(\{1,\dots,K\})$ and $\mathcal{K}^s\gets(\mathcal{K}_0,\dots,\mathcal{K}_s)$}
		\STATE{Apply \textsc{TreeEval}$((X^{\mathrm{NN}}_{j,s}, u_s(X^{\mathrm{NN}}_{j,s})), \mathcal{T})$ to get $K$-nearest neighbor under $d$}
		\STATE{Choose the $\mathcal{K}_s$-nearest neighbor and denote them by $(X^{\mathcal{K}^s}_{i_s,t_s}, U^{\mathcal{K}^s}_{i_s,t_s})$ with the corresponding reward $R^{\mathcal{K}^s}_{i_s,t_s}$}
		\STATE{Set $r\gets r + R^{\mathcal{K}^s}_{i_s,t_s}$ and $X^{\mathrm{NN}}_{j,s+1}\gets X^{\mathcal{K}^s}_{i_s,t_s+1}$ if $s<T$}
		\ENDFOR
		\STATE{$R\gets r + R$}
		\ENDFOR
	\end{algorithmic}
	\hspace*{\algorithmicindent} \textbf{Output:} $\frac{R}{l}$
\end{algorithm}

\begin{rem}
	\begin{enumerate}
		\item In our experiments instead of sampling independently from $\nu_0$, we randomly and uniformly select initial states from the data sets.
		\item For the LOB and LQR experiment, we slightly alter the algorithm to not allow reusing the same  transition twice within one nearest neighbor path. This does not change the complexity rates discussed below or the theoretical consistency and reduces the bias in these experiments. 
		\item In the experiments we use the standard Euclidean norm for nearest neighbor search in KNNR and MFMC for the LOB and LQR environment. In the BP environment we transform the state-action pair slightly before applying the Euclidean norm on the transformed state-action pair (see Appendix Section \ref{sec:exp_appendix}). 
		\item A further speed up could be possibly achieved with using approximate nearest neighbors methods, e.g. \cite{li2019approximate}. However, it is unclear whether the theoretical guarantees can be extended to this setting. We will leave this consideration for future investigations.
	\end{enumerate}
\end{rem}

\subsection{MFMC Algorithm}
Algorithm \ref{alg:mfmc} implements the model-free Monte Carlo algorithm proposed in \cite{fonteneau2010model,fonteneau2013batch}. Note in particular that the data set is different at each iteration on which the nearest neighbor is searched. Moreover, the change in the data depends on previous iterations.  Thus, one cannot parallelize the outer for-loop in Algorithm \ref{alg:mfmc}. This is in stark contrast to Algorithm \ref{alg:matching_ext} where the calculations in the outer for-loop can be conducted without interdependencies, making parallelization viable for trajectory generation using the KNNR algorithm.  

\begin{algorithm}
	\renewcommand{\thealgorithm}{MFMC}
	\caption{ Model-free Monte Carlo \cite{fonteneau2010model,fonteneau2013batch}}\label{alg:mfmc}
	\hspace*{\algorithmicindent} \textbf{Input:}  Data set $\mathcal{D}$; target policy $u$;  metric $d$ on $S$; number of resampled trajectories $l$
	\begin{algorithmic}[1]
		\STATE {$R\gets 0$}
		\STATE{$\widetilde{\mathcal{D}}\gets\mathcal{D}$}

		\FOR{$j$ from $1$ to $l$ }
		\STATE{$r\gets 0$}
		\STATE{Sample an initial state $X^{\mathrm{NN}}_{j,0}$ from $\nu_0$}
		\FOR{$s$ from $0$ to $T$}
		\STATE{Find the $1$-nearest neighbor of $(X^{\mathrm{NN}}_{j,s}, u_s(X^{\mathrm{NN}}_{j,s}))$ in $\widetilde{\mathcal{D}}$ and denote them by $(X_{i_s,t_s}, U_{i_s,t_s})$ with the corresponding reward $R_{i_s,t_s}$}
		\STATE{Set $r\gets r + R_{i_s,t_s}$, $\widetilde{\mathcal{D}}\gets\widetilde{\mathcal{D}}\setminus\{(X_{i_s,t_s}, U_{i_s,t_s},R_{i_s,t_s} )\} $, and $X^{\mathrm{NN}}_{j,s+1}\gets X_{i_s,t_s+1}$ if $s<T$}
		\ENDFOR
		\STATE{$R\gets r + R$}
		\ENDFOR
	\end{algorithmic}
	\hspace*{\algorithmicindent} \textbf{Output:} $\frac{R}{l}$
\end{algorithm}

\subsection{Complexity Analysis}
In this section, we discuss the complexity of the MFMC algorithm and the KNNR algorithm. For this purpose, recall that the computational complexity for a brute force $K$ nearest neighbour search with $n$ samples is $\mathcal{O}(1)$ at initialization and $\mathcal{O}(Kn)$ at evaluation.\footnote{Note that we omit dimension dependence for all complexities discussed here.} For tree-based methods the computational costs are $\mathcal{O}(Kn\log(n))$ at initialization and $\mathcal{O}(K\log(n))$ at evaluation (cf. Documentation of \cite{scikit-learn}).

We will make the following assumptions on the parameter choices: For both KNNR and MFMC we assume that the number of resampled paths growths with the same order as the data set i.e. $l_n=\mathcal{O}(n)$. If any slower rate is chosen for $l_n$, the variance bounds in \cite{fonteneau2010model} become suboptimal. In the experiments we use the rate $l_n = \lfloor 0.1n\rfloor$ for both the MFMC and the KNNR. For $K_n$ in KNNR, we assume $K_n=\mathcal{O}(n^{0.25})$ motivated by the empirical success in the experiments with continuous state-action spaces. As for the LQR and the LOB experiment we used the rate $K_n = \lfloor n^{0.25}\rfloor$. For the BP experiment, we got the best results with $K_n = 5 I_{n<10^4}+7I_{n=10^4}$. We believe that the improved performance with a differing rate is potentially explained by the discretization of the state-action space in the BP experiment.  

For the KNNR, we need to initialize the tree model at a cost of $\mathcal{O}(n^{1.25}\log(n))$ and search the tree for $K_n$ nearest neighbors $l_n\cdot T$ times yielding the rate $\mathcal{O}(n^{1.25}\log(n))$ for evaluation. Thus, overall  KNNR has the complexity $\mathcal{O}(n^{1.25}\log(n))$. The MFMC on the other hand looks for the nearest neighbour $l_n\cdot T$ times using a brute force nearest neighbour search yielding an overall rate of $\mathcal{O}(n^2)$. Note that a tree based method cannot improve this rate, since the data set is changing at each iteration. One would either need to fit $l_n\cdot T$ search trees or build one tree searching for $l_n\cdot T$ nearest neighbors. In either case the computational cost would be of the order $\mathcal{O}(n^2\log(n))$ and, thus, worse than brute force search.

%% file: 20230203experiments.tex
\section{Experiments}\label{sec:exp_appendix}

In this section, we give more details on the experiments we conducted and reported on in Section 4 of the main paper. In particular, we describe the three environments and comment on the implementation of the baselines.

\subsection{Linear Quadratic Regulator}
The first environment is the LQR problem proposed in \cite{recht2019tour} under the name of double integrator. We mainly follow the parametrization from \cite{recht2019tour} except that we choose a higher variance for the system noise. The environment is given as follows:

For $t=1,\dots,T$ with $T=10$, the dynamic model is given by 
\begin{equation}
X_{t+1}=A X_t+B \hat{u}_t+ e_t
\end{equation}
with the initial condition
\begin{equation}
	X_0=\left[\begin{array}{l}
		-1+0.5Z_1 \\
		Z_2
	\end{array}\right]
\end{equation}
where $X_t\in\mathbb{R}$ denotes the state, $\hat{u}_t\in\mathbb{R}$ the action, $Z_1,Z_2$ are independent standard normal random variables and 
\begin{equation}
A=\left[\begin{array}{ll}
	1 & 1 \\
	0 & 1
\end{array}\right],\quad B = \left[\begin{array}{l}
0 \\
1
\end{array}\right].
\end{equation}
Furthermore, for all $t$, $e_t$ is normal random vector with mean zero and covariance $10^{-2}I_2$  where $I_n$ is the $2$-dimensional identity matrix. The system noise  is independent over time and from the initial noise. The rewards are given by $r_t = X^{T}_tQX_t+R\hat{u}^2_t$ where
\begin{equation}
	Q=\left[\begin{array}{ll}
		1 & 0 \\
		0 & 0
	\end{array}\right] \text{ and } R=0.5.
\end{equation}
The behavior policy is given by 
\begin{equation}
	U_t=-K_t X_t
\end{equation}
with $K_t=[K_{t,1},K_{t,2}]^T$ where $K_{t,2}$ and $K_{t,2}$ are uniformly distributed on $[0,1]$ and $[1,2]$ respectively and independently for each $t$. As a target policy we use the policy given in \cite{recht2019tour}:
\begin{equation}
	u_t=-\widehat{K} X_t
\end{equation}
where
\begin{equation}
\widehat{K}=\left(R+B^T M B\right)^{-1} B^T M A
\end{equation}
and $M$ is the solution to the discrete algebraic Riccati equation
\begin{equation}
M=Q+A^T M A-\left(A^T M B\right)\left(R+B^T M B\right)^{-1}\left(B^T M A\right). 
\end{equation}

For reporting the performance on this environment in Figure 2 in the main paper, we run our KNNR algorithm and the baselines a hundred times each for each data set size. We then compare the estimated performances of the target policy with the actual performance that we obtain by simulating $10^6$ episodes sampled under the target policy.

\subsection{Optimal execution in Limit Order Book markets}
The next environment arises in the optimal execution of financial assets in limit order book markets (e.g. \cite{cartea2015algorithmic} Chapter 8.2). We give a detailed description of the continuous-time version of this model:
\begin{itemize}
	\item $\mathfrak{N}$ is the initial inventory of assets,
	\item $T$ the terminal time,
	\item $S=\left(S_{t}\right)_{0 \leq t \leq T}$ a stochastic process, modelling the evolution of the mid-price, given by $S_{t}=S_{0}+\sigma W_{t}, \sigma>0$, where $W=\left(W_{t}\right)_{0 \leq t \leq T}$ is a standard Brownian motion,
	\item $\delta=(\delta)_{0 \leq t \leq T}$ the difference between the mid-price and the offered price, which is the control variable,
	\item $M=\left(M_{t}\right)_{0 \leq t \leq T}$ a Poisson process with intensity $\lambda$ where $M_t$ denotes the number of arriving orders,
	\item $N^{\delta}=\left(N_{t}^{\delta}\right)_{0 \leq t \leq T}$ the (controlled) counting process corresponding to the number of assets from the original inventory which have been sold,
	\item $P(\delta)=e^{-\kappa \delta}$, with $\kappa>0$, the probability that an asset is being sold conditioned on an order arriving,
	\item $X^{\delta}=\left(X_{t}^{\delta}\right)_{0 \leq t \leq T}$ the reward process denoting the revenue of trade 
	$d X_{t}^{\delta}=\left(S_{t}+\delta_{t}\right) d N_{t}^{\delta}$,
	\item $Q_{t}^{\delta}=\mathfrak{N}-N_{t}^{\delta}$ the remaining assets in the inventory.
\end{itemize}
Our value of policy $\delta$ is given by 
\begin{equation}
V_{\delta} = \mathbb{E}\left[X_{\tau}^{\delta}+Q_{\tau}^{\delta}\left(S_{\tau}-\alpha Q_{\tau}^{\delta}\right) \mid X_{0^{-}}^{\delta}=0, S_{0}, Q_{0^{-}}^{\delta}=\mathfrak{N}\right]
\end{equation}
where $\alpha \geq 0$ determines the penalty for having inventory at time $T$ (via a linear impact function) and $\tau=T \wedge \min \left\{t: Q_{t}^{\delta}=0\right\}$ is a stopping time.
As a behavior policy we consider the random policy $\widehat{\delta}_t\sim\mathcal{V}[-0.01,0.04]$ independently sampled for each time step and as target policy we consider the (non-stationary) optimal policy given by
\begin{equation}
\delta_t(q)=\frac{1}{\kappa}\left[1+\log \frac{\sum_{n=0}^{q} \frac{\tilde{\lambda}^{n}}{n !} \exp(-\kappa \alpha(q-n)^{2})(T-t)^{n}}{\sum_{n=0}^{q-1} \frac{\tilde{\lambda}^{n}}{n !} \exp(-\kappa \alpha(q-1-n)^{2})(T-t)^{n}}\right]
\end{equation} 
where $\tilde{\lambda} =\lambda \exp(-1)$. 

For the implementation, we discretize this model using a standard Euler-Maryuama approach and use the following parameters
\begin{align*}
	&T=20 , \quad \lambda=50/60 , \quad \kappa=100, \quad \alpha=0.1  \quad \mathfrak{N}=3, \\
	&S_{0}= 30, \quad \Delta t = 1 \text { and } \quad \sigma= 0.1.
\end{align*}
where $ \Delta t$ denotes the discretization step-size.
\begin{rem}
Note that the LOB environment can get to a terminal state before the time limit is reached and that if the inventory is not equal to zero at the last point in time an additional reward is generated dependent deterministically on the remaining inventory. We adapt the baselines and the KNNR algorithm accordingly. The modification to account for this in the KNNR algorithm are minimal and following the consideration in Remark 1 of the main paper the estimation of the additional final reward is consistent.
\end{rem}

For reporting the performance on this environment in Figure 2 in the main paper, we run our KNNR algorithm and the baselines $30$ times each for each data set size. We then compare the estimated performances of the target policy with the actual performance that we obtain by simulating $10^6$ episodes sampled under the target policy. 

\subsection{Online Stochastic Bin Packing}
 
Next we consider a slightly altered version of the stochastic bin packing problem (BP) as introduced in \cite{balaji2019orl, gupta2012online}. For this model at each time period $t\in\{1,\dots,T\}$ one item arrives where the item is of some size $j\in\{1,\dots,J\}$ where $J,T\in\mathbb{N}$. At arrival the item needs to be packed in a bin of bin size $B\in \mathbb{N}$ where $J<B$. For that purpose an item $j$ can be packed in any existing bin that has free space of at least $B-j$ or in a newly opened bin. The number of bins that can be newly opened is not restricted. Let $b\in\{1,\dots,B\}$ denote all the possible levels of fill of a bin, then the state of this problem is of the form $X_t=((N_{b,t})_{b=1}^{B-1},j_t )$ where $N_b$ denotes the number of open bins filled to the level $b$ at time $t$ and $j_t$ the item arriving at time $t$. The action $a$ is encoded as an integer where $0$ signifies opening a new bin and any integer above zero marks a bin at level $b$ at which an item is placed. The only modification compared to \cite{balaji2019orl} is that we adapt the mechanism on how the item size is sampled at each period in time. We assume that $j_t$ is sampled from a Poisson distribution truncated between $1$ and $J$. The intensity parameter $\lambda$ of this distribution depends on $((N_{b,t})_{b=1}^{B-1})$. We calculate the average free space available over all opened bins and set $\lambda$ equal to this. If the average free space available over all opened bins is larger than the maximal item size or no bin is opened, we take $\lambda$ as the maximal item size. \footnote{In practice this mechanism leads to larger items arriving if more bins with  a lot space remaining are open. This behaviour would be consistent with the implementation of a dynamic pricing mechanism that increases prices as less capacity is available.} The reward $R_t$ in this system is calculated as the negative increment on the total unused capacity, meaning that if an item of size $j$ is placed in an existing bin the reward is given by the item size $j$. If a new bin needs to be opened the reward is $-(B-j)$. The quantity of interest is the expected sum of rewards over the episode:
 \begin{equation}
 	\mathbb{E}\Big[\sum_{t=1}^{T}R_t\Big]
 \end{equation}
 As a behavior policy we choose a random policy which at each time step  randomly samples uniformly from all feasible actions. As a target policy we use the sum of squares heuristic introduced and explained in \cite{balaji2019orl}. In our implementation, we use the following parameters
 \begin{align*}
 	T=30 , \quad J=9 \text { and } \quad B=10.
 \end{align*}

For reporting the performance on this environment in Figure 2 in the main paper, we run our KNNR algorithm and the baselines $30$ times each for each data set size. We then compare the estimated performances of the target policy with the actual performance that we obtain by simulating $0.5\cdot10^5$ episodes sampled under the target policy. 

\begin{rem}
Strictly speaking we do not use the Euclidean norm directly in the state-action space for the KNNR and MFMC algorithm, but we project the state-action space onto the bin counters and apply the Euclidean norm only on the bin counters for nearest neighbor search. To make this more explicit, introduce the following function $\pi: \mathbb{N}^{B+2}\rightarrow \mathbb{N}^{B}$ where $\pi(X_t,a_t)$ maps the current state and action to the tuple of open bins $(N_{t^+,b})_{b=1}^{B-1}$ after action $a_t$ has been taken. This means, if $a_t>0$, for $b=a_t$, $N_{t^+,b}=N_{t,b}-1$ and, if $a_t+j_t<B$, $N_{t^+,a_t+j_t}=N_{t,a_t+j_t}+1$. If $a_t=0$, $N_{t^+,j_t}=N_{t,j_t}+1$. In other words, we exploit that the instant change of the state due to the action is known and that only this modified state impacts the distribution of the next state. This is often the case in applications and makes our algorithm applicable to situations where the actions are not directly observable but the state is.   
\end{rem}

\subsection{Baselines}
Since the baselines used for our experiments are well-known, we will refrain from discussing the general approaches and only point out details in the implementation. The MFMC was discussed at length in a previous section. For the other baselines we generally refer to the Appendix of \cite{voloshin2021empirical}. For the PEIS and PDIS we note that these refer to the weighted versions of importance sampling. For PEIS, PDIS and WDR in the LOB and LQR setting, we randomize the target policy for calculating the  importance weights (since otherwise they would not be well-defined). We assume a normal distribution with the actual target policy as the mean and a constant standard deviation (LQR:0.01; LOB:0.005). For the FQE baselines we adapt them to the episodic setting i.e. Appendix of \cite{hao2021bootstrapping}. We consider two different regressions for the FQE. FQE Lin uses a linear regression with a linear-quadratic feature transformation of the state-action pair. FQE NN uses a nearest neighbor regression directly on the state-action pair. For the LOB and LQR setting we choose the number of nearest neighbors in the regression as $n^{0.25}$ and  for the BP environment as $n^{0.5}$ where $n$ denotes the number of episodes in the data set. These choices are based on empirical performance. For the WDR we use FQE Lin as the model component. Finally, the naive average of the rewards under the behavior policy (NA) is self-explanatory.